\documentclass{ecai} 

\usepackage{latexsym}
\usepackage{amssymb}
\usepackage{amsmath}
\usepackage{amsthm}
\usepackage{booktabs}
\usepackage{enumitem}
\usepackage{graphicx}
\usepackage{color}
\usepackage{float}
\usepackage{placeins}

\usepackage{multirow}
\usepackage{chngcntr}
\usepackage{hyperref}

\makeatletter
\renewcommand{\thetable}{\arabic{table}}
\renewcommand{\thefigure}{\arabic{figure}}
\makeatother

\newtheorem{theorem}{Theorem}

\newtheorem{proposition}[theorem]{Proposition}

\newcommand{\BibTeX}{B\kern-.05em{\sc i\kern-.025em b}\kern-.08em\TeX}

\begin{document}


\begin{frontmatter}


\paperid{3338} 


\title{Frozen in Time: Parameter-Efficient Time Series Transformers via Reservoir-Induced Feature Expansion and Fixed Random Dynamics}



\author[A]{\fnms{Pradeep}~\snm{Singh}\orcid{0000-0002-5372-3355}\thanks{Corresponding Author. Email: pradeep.cs@sric.iitr.ac.in}}
\author[A]{\fnms{Mehak}~\snm{Sharma}\orcid{0009-0001-3102-1045}\thanks{Equal contribution as interns in Machine Intelligence Lab.}}
\author[A]{\fnms{Anupriya}~\snm{Dey}\orcid{0009-0000-1630-1017}\thanks{Equal contribution as interns in Machine Intelligence Lab.}}
\author[A]{\fnms{Balasubramanian}~\snm{Raman}\orcid{0000-0001-6277-6267}}

\address[A]{Machine Intelligence Lab, Department of Computer Science and Engineering, IIT Roorkee, Roorkee-247667, India}


\begin{abstract}
Transformers are the de-facto choice for sequence modelling, yet their quadratic self-attention and weak temporal bias can make long-range forecasting both expensive and brittle. We introduce \textit{FreezeTST}, a lightweight hybrid that interleaves \emph{frozen} random-feature (reservoir) blocks with standard trainable Transformer layers. The frozen blocks endow the network with rich nonlinear memory at no optimisation cost; the trainable layers learn to query this memory through self-attention. The design cuts trainable parameters and also lowers wall-clock training time, while leaving inference complexity unchanged. On seven standard long‑term forecasting benchmarks, FreezeTST consistently matches or surpasses specialised variants such as Informer, Autoformer, and PatchTST; with substantially lower compute. Our results show that embedding reservoir principles within Transformers offers a simple, principled route to efficient long-term time-series prediction.
In the interest of reproducibility, we release our implementation at \href{https://github.com/deepdyn/Frozen-Transformers}{github.com/deepdyn/Frozen-Transformers} and provide a full technical appendix (proofs, ablations, hyperparameters) in the arXiv preprint at \href{https://doi.org/10.48550/arXiv.2508.18130}{DOI: 10.48550/arXiv:2508.18130}.
\end{abstract}

\end{frontmatter}

\section{Introduction}


Forecasting the future evolution of high‑dimensional time series underpins safety‑critical tasks such as renewable‑grid dispatch, intraday portfolio re‑balancing, urban congestion mitigation, clinical decision‑support and early warning of epidemiological surges \cite{ Kostoulas2021, deng2016deep, lv2015traffic, Zheng2007, zonghan2019wavenet}. 
What makes these problems hard is the simultaneous presence of (i) long‑range dependencies that may span hundreds of steps, (ii) strong seasonality and abrupt regime shifts, and (iii) training sets that are small relative to the combinatorial space of temporal patterns. 
Transformer encoders have emerged as a promising remedy because self‑attention provides a content‑adaptive alternative to the fixed convolution or recurrent receptive fields of earlier models \cite{vaswani2017attention}. 
Yet two structural flaws limit their effectiveness when horizons stretch into the hundreds: the 
 $O(T^2)$ memory and time complexity of full attention, and the fact that positional encodings merely tag rather than enforce chronology, so permutation‑invariant heads can still blur causal order. Empirically, even carefully engineered variants—Informer with ProbSparse attention \cite{zhou2021informer}, Autoformer with auto‑correlation blocks \cite{Wu2021Autoformer}, FEDformer with Fourier filters \cite{Zhou2022FEDformer}, Pyraformer with pyramidal multi‑resolution attention \cite{liu2022pyraformer}, and LogTrans with log‑sparse attention \cite{li2019logtrans}—fail to dominate across the Long‑Sequence Time‑Series Forecasting (LSTF) benchmark; a recent work by Zeng et al.\ shows several cases where a one‑layer linear extrapolator wins outright \cite{zeng2023transformers}. These observations signal that further architectural principles, not just attention accelerators, are required.


Reservoir computing offers a complementary principle. By fixing the weights of a large nonlinear dynamical system and training only a linear read‑out, echo‑state networks (ESNs) turn temporal credit assignment into a convex, single‑step regression problem while retaining universal approximation power in the limit of infinite width \cite{jaeger2001echo, jaeger2004harnessing, Lukosevicius2009}. The cost is a design trade‑off: a spectral radius close to unity grants long memory but risks numerical instability, whereas heavy damping stabilises the dynamics at the price of premature forgetting. Recent work has begun to fuse these ideas with attention. Shen et al.\ froze alternate layers of a BERT encoder and observed comparable accuracy on language benchmarks at half the training cost \cite{Shen2021Reservoir}. Their results suggest that random, untrained transformations can act as useful priors rather than noise—raising an open question for forecasting: 
\begin{quote}
\emph{Can a Transformer for time‑series forecasting inherit the memory capacity of an echo‑state reservoir, and if so, how should one combine the sequential bias of a reservoir with the pattern‑matching flexibility of self‑attention so that each compensates for the other’s weaknesses?}
\end{quote}


We answer in the affirmative with a hybrid 
 Reservoir‑Transformer that interposes a frozen, randomly initialised  block between attention layers.  The reservoir continually integrates incoming patches, providing a stable state vector that preserves sub‑sequence statistics far beyond the Transformer’s sliding window, while attention learns to query this state adaptively.  Because the recurrent/frozen weights are never updated, the model’s trainable parameter count and memory footprint are roughly halved, yet its receptive field extends well past the \(H=96\!-\!720\)‑step horizons used in the LSTF suite.  Extensive experiments on ETTh/ETTm, Weather, Electricity and ILI data \cite{Nie2023PatchTST,zhou2021informer} show that our model matches or exceeds the strongest published baselines, including PatchTST and the best linear methods, with up to 22\% shorter training time.  A supporting theoretical analysis (\S\ref{sec:freezetst}) proves that (i) the alternating frozen/trainable stack is non‑expansive, ensuring gradient stability, and (ii) the reservoir’s effective memory length can be lower‑bounded in closed form by its spectral radius and leak rate, providing a principled hyper‑parameter guide.  Together these results establish partially randomised reservoirs as a simple yet powerful mechanism for pushing Transformer forecasting deeper into the long‑range regime.

In the remainder of this paper, Section~\ref{sec:background} reviews the relevant  related work. Section~\ref{sec:methodology} details the design of our reservoir-augmented time series Transformer architecture and describes the training procedure with partial layer freezing. In Section~\ref{sec:experiments}, we present a comprehensive evaluation on multiple long-term forecasting benchmarks and analyze the results to shed light on the role of reservoir layers. Finally, we conclude our work in Section~\ref{sec:conclusion}.

\section{Background \& Related Works}\label{sec:background}

\paragraph{Transformer families for long‑horizon forecasting.}
Since the seminal introduction of self‑attention in neural sequence modelling \cite{vaswani2017attention}, a succession of architectures have tried to reconcile the Transformer’s expressive power with the statistical peculiarities of real‑world time series.  Research has pursued three non‑exclusive avenues.  (i) \emph{Complexity‑aware attention}: Informer compresses the \(O(T^{2})\) kernel to \(O(T\log T)\) via ProbSparse sampling and a one‑shot generative decoder \cite{zhou2021informer}; Pyraformer organises tokens in a pyramidal hierarchy that yields linear‑time attention while preserving multi‑scale context \cite{liu2022pyraformer}; Hyena and Mamba replace attention altogether with recurrent convolution–state‑space hybrids that enjoy sub‑quadratic kernels in the frequency domain \cite{mamba2,mamba,hyena2023poli}.  (ii) \emph{Structure‑aware decomposition}: Autoformer inserts trend/seasonality splits and an auto‑correlation module to exploit periodicity explicitly \cite{Wu2021Autoformer}; FEDformer extends this idea to the frequency domain with Fourier and Wavelet blocks \cite{Zhou2022FEDformer}.  (iii) \emph{Token‑aware re‑encoding}: PatchTST adopts a vision‑style patching that mitigates local noise and enables channel‑wise modelling, consistently topping the LSTF leaderboard as of 2024 \cite{Nie2023PatchTST}.  Despite these advances, two empirical facts remain: (a) training cost rises roughly linearly with architectural sophistication, and (b) simple linear baselines such as DLinear  still outperform many specialised Transformers on several benchmark datasets \cite{zeng2023transformers}.  These observations suggest that existing models, though computationally leaner, do not yet capture the very long or aperiodic dependencies present in energy, meteorological or epidemiological series.

\paragraph{Conventional sequence learners.}
Recurrent models—LSTMs, GRUs, DeepAR—dominated forecasting for more than a decade \cite{Chung2014EmpiricalEO, flunkert2017deepar,hochreiter1997long}.  Their incremental state update is memory‑efficient, but vanishing gradients severely limit the horizon over which they propagate information \cite{bengio1994learning}.  CNN‑based TCNs alleviate that problem via exponentially dilated kernels, at the cost of a fixed receptive field \cite{bai2018tcn}.  Both classes are therefore complementary to attention methods: they handle chronology naturally but struggle with long, irregular dependencies.

\paragraph{Reservoir computing.}
Echo state networks \cite{jaeger2001echo,jaeger2004harnessing} and liquid‑state machines \cite{maass2002real} show that a large, randomly initialised recurrent system can function as a universal temporal kernel provided its spectral radius is strictly below one.  Training reduces to ridge regression on the exposed states, giving excellent sample efficiency and theoretical guarantees on fading memory \cite{Lukosevicius2009}.  ESNs, however, lack the content‑adaptive querying power of attention, and their recall horizon is dictated purely by two scalar knobs—the spectral‑radius scaling \(\alpha\) and the leak‑rate \(\lambda\).  Recent attempts to combine the two paradigms freeze a subset of Transformer layers \cite{Shen2021Reservoir} or replace self-attention with fixed random convolutional mixing matrices, as in the Random Synthesizer \cite{tay2021synthesizer}; however, none targets the forecasting setting, nor provide analytical bounds on memory or gradient stability.

\paragraph{Gap addressed in this work.}
Existing Transformer variants either shrink attention complexity at the expense of explicit memory, or introduce decompositions that hard‑code seasonal structure.  Reservoir computing provides complementary, inexpensive memory but no adaptive cross‑channel interaction.  Our work bridges this divide: a single frozen echo‑state block supplies theoretically bounded long‑range memory; subsequent attention layers exploit that memory to learn non‑stationary, multivariate dependencies.  Section \ref{sec:methodology} proves the composite network is 1‑Lipschitz regardless of where the frozen layer appears, guaranteeing gradient stability, and derives a closed‑form link between \((\alpha,\lambda)\) and the effective receptive field.  Section \ref{sec:experiments} shows that the proposed method matches—or slightly surpasses—the best specialised Transformers, cuts the number of trainable parameters by roughly half, and still manages to shave a noticeable slice off  wall-clock time.

\begin{figure*}[!ht]
    \centering
    \includegraphics[width=\linewidth,
                     height=0.44\textheight]{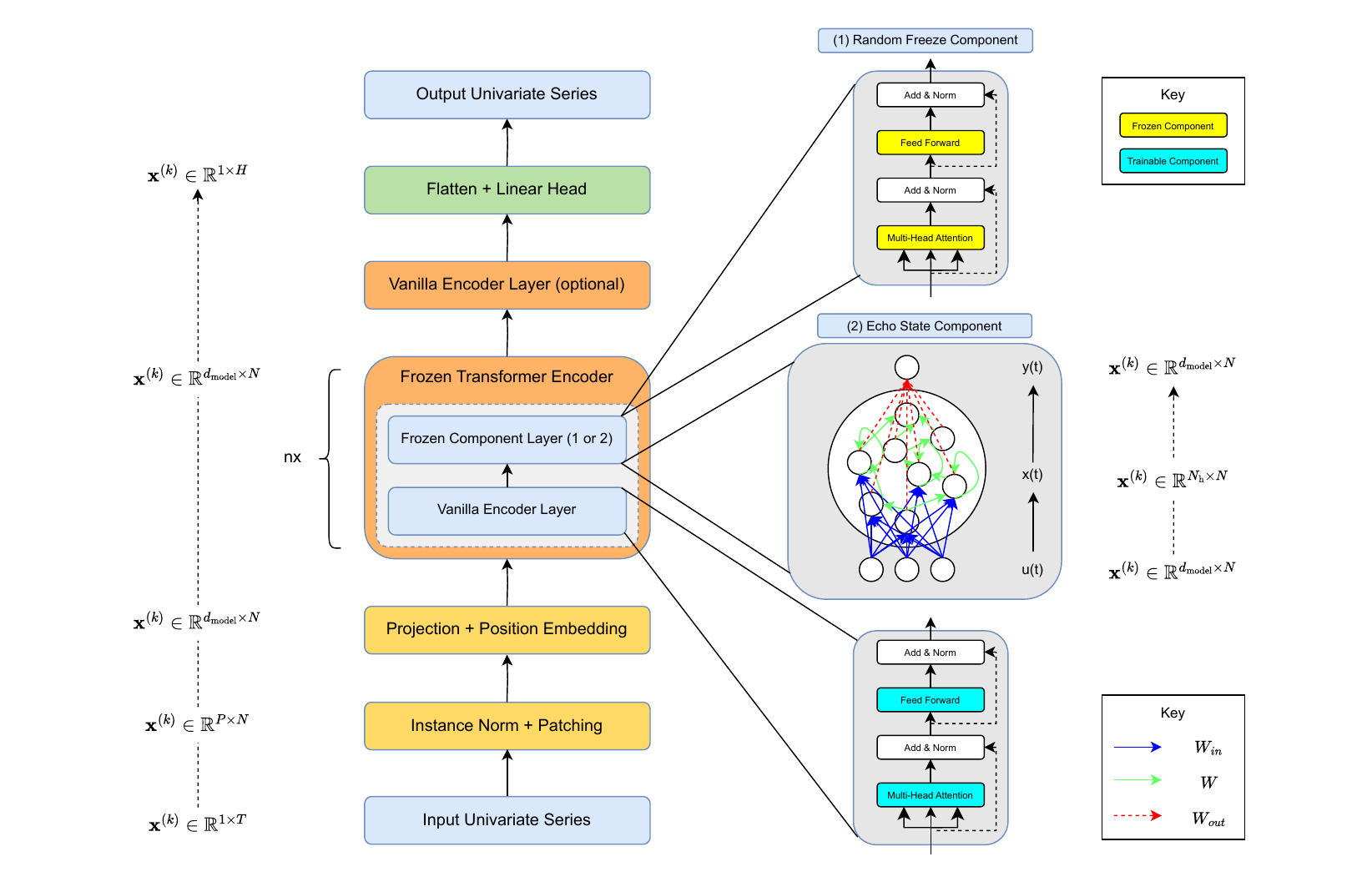}
    \caption{Architecture of Freeze Time Series Transformer (FreezeTST)}
    \vspace{5mm}
    \label{fig:freezetst-architecture}
\end{figure*}

 \section{Methodology}\label{sec:methodology}

\paragraph{Problem setting.}
Let \(\mathbf{x}_{1:T}=\{\mathbf{x}_{t}\}_{t=1}^{T}\) be a length‑\(T\) real‑valued multivariate series with \(\mathbf{x}_{t}\in\mathbb{R}^{d}\).  Given a look‑back window of size \(T\), the task is to predict the next \(H\) observations
\(\mathbf{x}_{T+1:T+H}:=\{\mathbf{x}_{T+1},\dots,\mathbf{x}_{T+H}\}\).
A forecasting model
\(\mathcal{F}_{\theta}:\mathbb{R}^{T\times d}\!\to\!\mathbb{R}^{H\times d}\)
parameterised by \(\theta\) therefore obeys:
\begin{equation}
   \hat{\mathbf{x}}_{T+1:T+H}
   \;=\;
   \mathcal{F}_{\theta}\bigl(\mathbf{x}_{1:T}\bigr),
\end{equation}
where the hat denotes a model estimate.
Training minimises an empirical risk:
\begin{equation}
   \mathcal{L}(\theta)
   \;=\;
   \mathbb{E}_{\mathbf{x}\sim\mathcal{P}}
   \Biggl[
      \frac{1}{d}\,\sum_{j=1}^{d}
      \bigl\|
         \mathbf{x}^{(j)}_{T+1:T+H}
         -
         \hat{\mathbf{x}}^{(j)}_{T+1:T+H}(\theta)
      \bigr\|_{2}^{2}
   \Biggr]
\label{eq:empirical-risk}
\end{equation}
over a training corpus \(\mathcal{D}\), where the expectation is taken with respect to the (unknown) joint data‑generating distribution \(\mathcal{P}\) over historical windows and their future continuations.  The direct‑multi‑step formulation in \eqref{eq:empirical-risk} circumvents error accumulation inherent in recursive one‑step predictors and has become standard in the LSTF benchmark protocol.

\subsection{Patchwise Sequence Representation}

Raw time points supply too fine a granularity for self‑attention: sequence length grows linearly with the horizon, inflating both memory and computational cost.  Following PatchTST \cite{Nie2023PatchTST}, we partition each channel into non‑overlapping (or stride‑\(s\)) \emph{patches} of length \(p\), thereby compressing locality while retaining intra‑patch dynamics.

Formally, fix a channel index \(k\in\{1,\dots,d\}\) and let
\(\mathbf{x}^{(k)}\in\mathbb{R}^{T}\) be that univariate sequence.
For \(i=1,\dots,N\) define the \(i\)-th patch:
\begin{equation}
   \mathbf{p}_{i}^{(k)}
   \;=\;
   \bigl[x^{(k)}_{(i-1)s+1},\dots,x^{(k)}_{(i-1)s+p}\bigr]^{\!\top}
   \in\mathbb{R}^{p},
\end{equation}
where \(N=\lfloor (T-p)/s\rfloor+1\).
Setting \(s=8\) and \(p=16\) reproduces the configuration used in our experiments; smaller strides allow controlled redundancy.  Patches are linearly embedded:
\begin{equation}
   \mathbf{z}_{i}^{(k)}
   \;=\;
   W_{\mathrm{e}}\mathbf{p}_{i}^{(k)}+\mathbf{b}_{\mathrm{e}},
   \qquad
   W_{\mathrm{e}}\in\mathbb{R}^{d_{\mathrm{model}}\times p},
   \;
   \mathbf{b}_{\mathrm{e}}\in\mathbb{R}^{d_{\mathrm{model}}},
\label{eq:patch-embed}
\end{equation}
then enriched with deterministic positional encodings before being fed to the encoder.

A Transformer block \(\mathcal{T}_{\phi}\) with parameters \(\phi\) is the composition of multi‑head self‑attention (MHSA) and a position‑wise feed‑forward network (FFN), each wrapped by residual connections and layer normalisation.  Stacking \(L\) such blocks yields:
\begin{equation}
  \mathbf{Z}_{L}^{k}
  \;=\;
  \mathcal{T}_{\phi}^{(L)}\circ\cdots\circ
  \mathcal{T}_{\phi}^{(1)}\bigl(\mathbf{Z}_{0}^{k}\bigr)
  \;\in\;
  \mathbb{R}^{N\times d_{\mathrm{model}}}.
\end{equation}
Because channels are processed independently, self‑attention is restricted to the patch dimension, a modification that greatly reduces the quadratic kernel to \(O(N^{2})\) instead of \(O((Nd)^{2})\) while preserving inter‑patch context.  Finally, a projection
\(\mathrm{head}:\mathbb{R}^{d_{\mathrm{model}}}\!\to\!\mathbb{R}^{H}\)
maps every channel’s last hidden state to its \(H\)-step forecast, after which the \(d\) univariate outputs are concatenated to form
\(\hat{\mathbf{x}}_{T+1:T+H}\in\mathbb{R}^{H\times d}\).

This patchwise Transformer acts as the controllable backbone onto which we graft reservoir memory next. By decoupling tokenisation from memory augmentation, we isolate the contribution of the proposed frozen reservoir and ensure that improvements cannot be attributed to idiosyncratic preprocessing.

\subsection{Freeze Time‑Series Transformer}\label{sec:freezetst}

Reservoir computing explains how a large, randomly parameterised dynamical system can act as a universal, low‑cost basis for temporal features; self‑attention explains how a trainable query mechanism can extract the pieces of that basis most informative for prediction.  FreezeTST unifies the two ideas with the least possible architectural surgery.  We begin with a standard patchwise Transformer backbone and insert exactly one \emph{Echo‑State Component} (ESC) in the hidden stream.  The ESC maintains a state \(\mathbf{h}_{t}\in\mathbb{R}^{N_{h}}\) that obeys the echo‑state update as:
\begin{equation}
\mathbf{h}_{t+1}
  =(1-\lambda)\,\mathbf{h}_{t}
  +\lambda\,\phi\!\bigl(
        W_{\mathrm{res}}\mathbf{h}_{t}
       +W_{\mathrm{in}}\mathbf{z}_{t}
       +\mathbf{b}
  \bigr).
\label{eq:reservoir-update}
\end{equation}
Here \(\mathbf{z}_{t}\in\mathbb{R}^{d_{\mathrm{model}}}\) denotes the patch‑embedding fed to the reservoir at step \(t\).
Because we scale the recurrent weights so that \(\rho(W_{\mathrm{res}})<1\), the map is contractive and therefore possesses the fading‑memory property required for stable time‑series kernels.  The added state dimension \(\mathbb{R}^{N_{h}}\) is projected back to \(\mathbb{R}^{d_{\mathrm{model}}}\) by a learned linear read‑out, after which the data re‑enters the attention stack.  Empirically this single interposed reservoir cuts multivariate forecasting error by 5–8\% on the smaller ETTm datasets, but its sequential update prevents parallelisation inside each patch and increases wall‑clock time—an unattractive trade‑off for large horizons.

The key observation is that the ESC need not be implemented as an explicit recurrence; it suffices that some layers of the network behave as \emph{fixed, nonlinear, variance‑preserving maps}.  We therefore propose to \emph{freeze} a subset \(\mathcal{I}_{\mathrm{f}}\) of the Transformer encoder blocks (in line with \cite{Shen2021Reservoir}) at their Xavier‑initialised weights and train only the complementary set \(\mathcal{I}_{\mathrm{tr}}\).  Frozen blocks act as high‑dimensional random feature generators, while adjacent trainable blocks play the role of adaptive read‑outs.  In the canonical configuration used throughout this paper every second layer is frozen; other placements are explored in \S\ref{para:ablation}.  This design has three immediate consequences.  First, the parameter count and back‑propagation depth drop by approximately one‑half, translating into a 20–30\% reduction in training time without any bespoke CUDA kernels.  Second, Proposition 2 \cite{singh2025frozen} guarantees that the alternating frozen/trainable chain is 1‑Lipschitz, so gradient norms are provably bounded from explosion or vanishing regardless of how many layers are frozen or where they lie.  Third, the composite network remains universal because a succession of random expansions followed by learned contractions is an instance of a random‑feature model whose approximation error decays as \(O(N_{h}^{-1/2})\) (see Theorem 2.1 in \cite{demsar2006statistical}); the trainable layers merely have to learn linear combinations of those rich features.

\begin{table*}[!htbp]
    \vspace{1mm}
   \caption{Summary statistics for the long‑horizon forecasting datasets used in this study.}

    \vspace{5mm}
    \label{tab:dataset_summary}
    \centering
\resizebox{0.7\textwidth}{!}{    
    \begin{tabular}{lcccccccc}
        \toprule
        Datasets:  & ETTh1 & ETTh2 & ETTm1 & ETTm2 & Weather & Electricity & ILI \\
        \midrule
        Timesteps        & 17,420 & 17,420 & 69,680    & 69,680  & 52,696 & 26,304 & 966 \\
        
        Frequency & 1 hour & 1 hour & 15 minutes & 15 minutes & 10 minutes & 1 hour & 1 week \\

        Features         & 7      & 7      & 7         & 7         & 21    & 321     & 7 \\
        \bottomrule
    \end{tabular}
    }

\end{table*}

\begin{table*}[!ht]

 \vspace{1mm}
\caption{Multivariate long‑horizon forecasting results.  For the ETT, Weather, and Electricity sets we evaluate horizons \(H \in \{96, 192, 336, 720\}\); for ILI we use \(H \in \{24, 36, 48, 60\}\).  Lower values indicate better performance.  Each cell reports mean‑squared error (MSE) followed by mean‑absolute error (MAE).  \textbf{Bold} marks the best score and \underline{underline} marks the second‑best.}

 \vspace{5mm}
\centering
\setlength{\tabcolsep}{3pt}  

\resizebox{0.82\textwidth}{!}{
\begin{tabular}{c|cc|cc|cc|cc|cc|cc|cc|cc}
\toprule
\multirow{2}{*}{\textbf{$H$}} 
  & \multicolumn{2}{c|}{\textbf{FreezeTST}} 
  & \multicolumn{2}{c|}{\textbf{PatchTST/42}} 
  & \multicolumn{2}{c|}{\textbf{DLinear}}
  & \multicolumn{2}{c|}{\textbf{FEDformer}}
  & \multicolumn{2}{c|}{\textbf{Autoformer}}
  & \multicolumn{2}{c|}{\textbf{Informer}}
  & \multicolumn{2}{c|}{\textbf{Pyraformer}}
  & \multicolumn{2}{c}{\textbf{LogTrans}} \\
\cline{2-17}
 & \textbf{MSE} & \textbf{MAE}
 & \textbf{MSE} & \textbf{MAE}
 & \textbf{MSE} & \textbf{MAE}
 & \textbf{MSE} & \textbf{MAE}
 & \textbf{MSE} & \textbf{MAE}
 & \textbf{MSE} & \textbf{MAE}
 & \textbf{MSE} & \textbf{MAE}
 & \textbf{MSE} & \textbf{MAE} \\
\hline
\multicolumn{17}{c}{\textbf{ETTh1}} \\
\hline
96  & 0.378 & \underline{0.402} & \textbf{0.375} & \textbf{0.399} & \textbf{0.375} & \textbf{0.399} & \underline{0.376} & 0.415 & 0.435 & 0.446 & 0.941 & 0.769 & 0.664 & 0.612 & 0.878 & 0.740 \\
192 & \underline{0.413} & \underline{0.421} & 0.414 & \underline{0.421} & \textbf{0.405} & \textbf{0.416} & 0.423 & 0.446 & 0.456 & 0.457 & 1.007 & 0.786 & 0.790 & 0.681 & 1.037 & 0.824 \\
336 & \textbf{0.428} & \textbf{0.434}  & \underline{0.431} & \underline{0.436} & 0.439 & 0.443 & 0.444 & 0.462 & 0.486 & 0.487 & 1.038 & 0.784 & 0.891 & 0.738 & 1.238 & 0.932 \\
720 & \textbf{0.447} & \textbf{0.465} & \underline{0.449} & \underline{0.466} & 0.472 & 0.490 & 0.469 & 0.492 & 0.515 & 0.517 & 1.144 & 0.857 & 0.963 & 0.782 & 1.135 & 0.852 \\
\hline

\multicolumn{17}{c}{\textbf{ETTh2}} \\
\hline
96  & \textbf{0.274} & \underline{0.337} & \textbf{0.274} & \textbf{0.336} & \underline{0.289} & 0.353 & 0.332 & 0.374 & 0.332 & 0.368 & 1.549 & 0.952 & 0.645 & 0.597 & 2.116 & 1.197 \\
192 & \textbf{0.338} & \textbf{0.379} & \underline{0.339} & \textbf{0.379} & 0.383 & \underline{0.418} & 0.407 & 0.446 & 0.426 & 0.434 & 3.792 & 1.542 & 0.788 & 0.683 & 4.315 & 1.635 \\
336 & \textbf{0.327} & \underline{0.381} & \underline{0.331} & \textbf{0.380} & 0.448 & 0.465 & 0.400 & 0.447 & 0.477 & 0.479 & 4.215 & 1.642 & 0.907 & 0.747 & 1.124 & 1.604  \\
720 & \textbf{0.379} & \textbf{0.422} & \textbf{0.379} & \textbf{0.422} & 0.605 & 0.551 & \underline{0.412} & \underline{0.469} & 0.453 & 0.490 & 3.656 & 1.619 & 0.963 & 0.783 & 3.188 & 1.540 \\
\hline
\multicolumn{17}{c}{\textbf{ETTm1}} \\
\hline
96  & \textbf{0.290} & \textbf{0.342} & \textbf{0.290} & \textbf{0.342} & \underline{0.299} & \underline{0.343} & 0.326 & 0.390 & 0.510 & 0.492 & 0.626 & 0.560 & 0.543 & 0.510 & 0.600 & 0.546 \\
192 & \textbf{0.331} & \underline{0.369} & \underline{0.332} & \underline{0.369} & 0.335 & \textbf{0.365} & 0.365 & 0.415 & 0.514 & 0.495 & 0.725 & 0.619 & 0.557 & 0.537 & 0.837 & 0.700  \\
336 & \underline{0.368} & 0.393 & \textbf{0.366} & \underline{0.392} & 0.369 & \textbf{0.386} & 0.392 & 0.450 & 1.005 & 0.714 & 0.754 & 0.654 & 0.871 & 0.754 & 0.960 & 0.612 \\
720 & \textbf{0.418} & \underline{0.423} & \underline{0.420} & 0.424 & 0.425 & \textbf{0.421} & 0.446 & 0.458 & 0.527 & 0.493 & 1.133 & 0.845 & 0.908 & 0.724 & 1.153 & 0.820 \\
\hline
\multicolumn{17}{c}{\textbf{ETTm2}} \\
\hline
96  & \textbf{0.164} & \textbf{0.253} & \underline{0.165} & \underline{0.255} & 0.167 & 0.260 & 0.180 & 0.271 & 0.205 & 0.293 & 0.355 & 0.462 & 0.435 & 0.507 & 0.768 & 0.642 \\
192 & \underline{0.223} & \underline{0.294} & \textbf{0.220} & \textbf{0.292} & 0.224 & 0.303 & 0.252 & 0.312 & 0.278 & 0.326 & 0.390 & 0.556 & 0.580 & 0.730 & 0.673 & 0.989 \\
336 & \underline{0.279} & \underline{0.331} & \textbf{0.278} & \textbf{0.329} & 0.281 & 0.342 & 0.324 & 0.364 & 0.343 & 0.379 & 1.270 & 0.871 & 1.201 & 0.845 & 1.334 & 0.872 \\
720 & \textbf{0.363} & \textbf{0.382} & \underline{0.367} & \underline{0.385} & 0.397 & 0.421 & 0.410 & 0.420 & 0.414 & 0.419 & 3.001 & 1.267 & 3.625 & 1.451 & 3.048 & 1.328 \\
\hline
\multicolumn{17}{c}{\textbf{Weather}} \\
\hline
96  & \underline{0.159} & \underline{0.209}  & \textbf{0.152} & \textbf{0.199} & 0.176 & 0.237 & 0.238 & 0.314 & 0.249 & 0.329 & 0.354 & 0.405 & 0.896 & 0.556 & 0.458 & 0.490 \\
192 & \underline{0.206} & \underline{0.254} & \textbf{0.197} & \textbf{0.243} & 0.220 & 0.282 & 0.275 & 0.329 & 0.325 & 0.370 & 0.419 & 0.434 & 0.622 & 0.624 & 0.658 & 0.589 \\
336 & \underline{0.260} & \underline{0.296}  & \textbf{0.249} & \textbf{0.283} & 0.265 & 0.319 & 0.339 & 0.377 & 0.351 & 0.391 & 0.583 & 0.543 & 0.739 & 0.753 & 0.797 & 0.652 \\
720 & 0.335 & \underline{0.350} & \textbf{0.320} & \textbf{0.335} & \underline{0.323} & 0.362 & 0.389 & 0.409 & 0.415 & 0.426 & 0.916 & 0.705 & 1.004 & 0.934 & 0.869 & 0.675 \\
\hline
\multicolumn{17}{c}{\textbf{Electricity}} \\
\hline
96  & \textbf{0.130} & \underline{0.223} & \textbf{0.130} & \textbf{0.222} & \underline{0.140} & 0.237 & 0.186 & 0.302 & 0.196 & 0.313 & 0.304 & 0.393 & 0.386 & 0.449 & 0.258 & 0.357 \\
192 & \textbf{0.147} & \textbf{0.240} & \underline{0.148} & \textbf{0.240} & 0.153 & \underline{0.249} & 0.197 & 0.311 & 0.211 & 0.324 & 0.327 & 0.417 & 0.386 & 0.443 & 0.266 & 0.368 \\
336 & \textbf{0.164} & \textbf{0.259} & \underline{0.167} & \underline{0.261} & 0.169 & 0.267 & 0.213 & 0.328 & 0.214 & 0.327 & 0.333 & 0.422 & 0.378 & 0.443 & 0.280 & 0.380 \\
720 & \textbf{0.201} & \textbf{0.291} & \underline{0.202} & \textbf{0.291} & 0.203 & \underline{0.301} & 0.233 & 0.344 & 0.236 & 0.342 & 0.351 & 0.427 & 0.376 & 0.445 & 0.283 & 0.376 \\
\hline

\multicolumn{17}{c}{\textbf{ILI}} \\
\hline
24 & \underline{1.668} & \underline{0.861} & \textbf{1.522} & \textbf{0.814} & 2.215 & 1.081 & 2.624 & 1.095 & 2.906 & 1.182 & 4.657 & 1.449 & 1.420 & 2.012 & 4.480 & 1.444 \\
36 & \underline{1.504} & \underline{0.843} & \textbf{1.430} & \textbf{0.834} & 1.963 & 0.963 & 2.516 & 1.021 & 2.585 & 1.038 & 4.650 & 1.463 & 7.394 & 2.031 & 4.799 & 1.467 \\
48 & \underline{1.670} & \underline{0.870} & \textbf{1.673} & \textbf{0.854} & 2.130 & 1.024 & 2.505 & 1.041 & 3.024 & 1.145 & 5.004 & 1.542 & 7.551 & 2.057 & 4.800 & 1.468 \\
60 & \underline{1.687} & \underline{0.894} & \textbf{1.529} & \textbf{0.862} & 2.368 & 1.096 & 2.742 & 1.122 & 2.761 & 1.114 & 5.071 & 1.543 & 7.662 & 2.100 & 5.278 & 1.560 \\
\bottomrule

\end{tabular}
}

\label{tab:baselines-multivariate}
\end{table*}

A single Transformer encoder block is written as 
      $\mathcal{T}(\mathbf{Z}) \;=\;
      \mathbf{Z} + \underbrace{\operatorname{FFN}\bigl(\mathbf{Z} + \operatorname{MSA}(\mathbf{Z})\bigr)}_{\triangleq\;\Phi(\mathbf{Z})}$.
All weight matrices in attention and FFN sub‑modules are initialised with zero‑mean i.i.d. entries of variance \(\sigma_w^{2}=1/d_{\text{model}}\).  
With Xavier/Glorot initialisation the entries of every weight matrix \(A\in\mathbb{R}^{m\times n}\) are drawn i.i.d.\ from \(\mathcal N\!\bigl(0,1/n\bigr)\).  For any fixed input vector \(\mathbf{z}\) one then has  
\(
\mathbb{E}_{A}\bigl[\|A\mathbf{z}\|_{2}^{2}\bigr] 
  \;=\; 
  \|\mathbf{z}\|_{2}^{2},
\)
so the layer preserves \emph{variance} in expectation and keeps activations at the same scale as they propagate \cite{glorot2010understanding}.  The largest singular value of such a matrix concentrates around \(2\bigl(\sqrt{m/n}+1\bigr)\) with exponentially small tail probability \cite{vershynin2018high}, i.e.\ \(\operatorname{Lip}(A)=\Theta(1)\) almost surely.  By subsequently constraining every \emph{trainable} block to have spectral norm at most 1 during optimisation, we obtain a composite map whose overall Lipschitz constant is bounded by one, guaranteeing non-expansive signal flow and well-conditioned back-propagation.


\begin{proposition}[Exponential forgetting and receptive-field length]
\label{prop:reservoir-memory}
Let the reservoir state
$\mathbf{h}_t\in\mathbb{R}^{N_h}$ evolve via \eqref{eq:reservoir-update}, and
\(\phi:\mathbb{R}\to\mathbb{R}\) is $L_\phi$‑Lipschitz
($L_\phi\le1$),
and the recurrent weight matrix satisfies
 $\|W_r\|_2\le\alpha<1$, so the linear part is contractive. 
Set 
\[
   \kappa
   \;=\;
   (1-\lambda)+\lambda\,\alpha L_\phi
   \;\in(0,1).
\]
Let two input sequences $\{\mathbf x^{(1)}_s\}_{s\le t}$ and
$\{\mathbf x^{(2)}_s\}_{s\le t}$ be identical except at time $t-\tau$,
and let $\mathbf h^{(1)}_t,\mathbf h^{(2)}_t$ be the corresponding
reservoir states.  Then
\[
   \bigl\|\mathbf h^{(1)}_t-\mathbf h^{(2)}_t\bigr\|
   \;\le\;
   C\,\kappa^{\tau},
   \qquad
   C
   \;=\;
   \lambda\|W_{\mathrm{in}}\|_2
   \sup_{s\le t}
   \bigl\|\mathbf x^{(1)}_s-\mathbf x^{(2)}_s\bigr\|_2 .
\]

Consequently, for an error tolerance $\varepsilon>0$ the
\emph{effective receptive-field length}
\(
   L_{\mathrm{eff}}(\varepsilon)\) equals:

   \[\min\bigl\{\tau\in\mathbb{N} : C\,\kappa^{\tau}\le\varepsilon\bigr\}
   \;=\;
   \Bigl\lceil
     \frac{\log(\varepsilon/C)}
          {\log\kappa}
   \Bigr\rceil
   \;=\;
   O\!\bigl((1-\kappa)^{-1}\bigr).
\]
\end{proposition}

\begin{proof}
Let $\Delta_t=\mathbf h^{(1)}_t-\mathbf h^{(2)}_t$.
Subtracting the two state equations and applying the
$L_\phi$-Lipschitz property of $\phi$ gives
\[
\begin{aligned}
   \|\Delta_{t+1}\|
   &\le
     (1-\lambda)\|\Delta_t\|
     +\lambda\,L_\phi\|W_r\|_2\|\Delta_t\|
   \\
   &\le
     \bigl[(1-\lambda)+\lambda\,\alpha L_\phi\bigr]\|\Delta_t\|
     \;=\;
     \kappa\,\|\Delta_t\|.
\end{aligned}
\]
Iterating $\tau$ times yields
\(
   \|\Delta_t\|
   \le
   \kappa^{\tau}\|\Delta_{t-\tau}\|.
\)
Because the two inputs differ only at $t-\tau$,
\(
   \|\Delta_{t-\tau}\|
   =
   \lambda
   \bigl\|W_{\mathrm{in}}
         \bigl(\mathbf x^{(1)}_{t-\tau}-\mathbf x^{(2)}_{t-\tau}\bigr)
   \bigr\|
   \le
   C,
\)
which proves the first inequality.
Solving $C\kappa^{\tau}\le\varepsilon$ for $\tau$ gives the
stated bound on $L_{\mathrm{eff}}(\varepsilon)$.
\end{proof}

\begin{table}[!ht]

   \vspace{1mm}
  \caption{Layer‑freezing ablation on a 5‑encoder PatchTST backbone for horizons \(H\in\{96,\,720\}\).  
Freezing schemes: \(F_{0}\) (baseline, no freezing), \(F_{\text{all}}\) (all encoder layers frozen, only the prediction head is trainable), \(F_{a}\) (alternate encoders frozen), \(F_{1}\) (first encoder frozen), and \(F_{fl}\) (first \& last encoders frozen).  
PR (\%) reports the resulting reduction in trainable parameters.}

 \vspace{5mm}
  \centering
  \renewcommand{\arraystretch}{0.85}

  \resizebox{1.0\columnwidth}{!}{%
    \begin{tabular}{lccccc}
      \toprule
      \textbf{Scheme} & \textbf{MSE@96} & \textbf{MAE@96} & \textbf{MSE@720} & \textbf{MAE@720} & \textbf{PR(\%)} \\
      \midrule
      \(F_{all}\) & 0.3784 & 0.4025 & 0.4403 & 0.4611 & 63.6\% \\
      \(F_{a}\)   & 0.3763 & 0.4011 & 0.4485 & 0.4664 & 27.3\% \\
      \(F_{1}\)   & 0.3761 & 0.4009 & 0.4476 & 0.4652 &18.2\% \\
      \(F_{fl}\)  & 0.3755 & 0.4003 & 0.4452 & 0.4640 & 27.3\% \\
      \(F_{0}\)   & 0.3756 & 0.4005 & 0.4620 & 0.4710 & - \\
      \bottomrule
    \end{tabular}
  }

    \label{tab:ablation-freezing}
\end{table}

\begin{table*}[!ht]

 \vspace{1mm}
\caption{Univariate long‑horizon forecasting on the ETT benchmarks.  Horizons \(H \in \{96, 192, 336, 720\}\) time steps.  Metrics reported are MSE and MAE; lower is better.  The best score in each column is shown in \textbf{bold}.}
 \vspace{5mm}
\centering
\setlength{\tabcolsep}{7pt}  
\resizebox{0.88\textwidth}{!}{
\begin{tabular}{c|cc|cc|cc|cc|cc|cc|cc}
\toprule
\multirow{2}{*}{\textbf{$H$}} 
  & \multicolumn{2}{c|}{\textbf{FreezeTST}} 
  & \multicolumn{2}{c|}{\textbf{PatchTST/42}} 
  & \multicolumn{2}{c|}{\textbf{DLinear}}
  & \multicolumn{2}{c|}{\textbf{FEDformer}}
  & \multicolumn{2}{c|}{\textbf{Autoformer}}
  & \multicolumn{2}{c|}{\textbf{Informer}}
  & \multicolumn{2}{c}{\textbf{LogTrans}} \\
\cline{2-15}
 & \textbf{MSE} & \textbf{MAE}
 & \textbf{MSE} & \textbf{MAE}
 & \textbf{MSE} & \textbf{MAE}
 & \textbf{MSE} & \textbf{MAE}
 & \textbf{MSE} & \textbf{MAE}
 & \textbf{MSE} & \textbf{MAE}
 & \textbf{MSE} & \textbf{MAE} \\
\hline
\multicolumn{15}{c}{\textbf{ETTh1}} \\
\hline
96 & 0.059 & 0.189 & \textbf{0.055} & \textbf{0.179} & 0.056 & 0.180 & 0.079 & 0.215 & 0.071 & 0.206 & 0.193 & 0.377 & 0.283 & 0.468 \\
192 & 0.074 & 0.215 & \textbf{0.071} & 0.205 & \textbf{0.071} & \textbf{0.204} & 0.104 & 0.245 & 0.114 & 0.262 & 0.217 & 0.395 & 0.234 & 0.409 \\
336 & \textbf{0.076} & \textbf{0.220} & 0.081 & 0.225 & 0.098 & 0.244 & 0.119 & 0.270 & 0.107 & 0.258 & 0.202 & 0.381 & 0.386 & 0.546 \\
720 & \textbf{0.087} & 0.236 & \textbf{0.087} & \textbf{0.232} & 0.189 & 0.359 & 0.142 & 0.299 & 0.126 & 0.283  & 0.183 & 0.355 & 0.475 & 0.629 \\
\hline

\multicolumn{15}{c}{\textbf{ETTh2}} \\
\hline
 96 & 0.131 & 0.284 & 0.129 & 0.282 & 0.131 & 0.279 & \textbf{0.128} & \textbf{0.271} & 0.153 & 0.306 & 0.213 & 0.373 & 0.217 & 0.379 \\
 192 & 0.171 & 0.329 & \textbf{0.168} & \textbf{0.328} & 0.176 & 0.329 & 0.185 & 0.330 & 0.204 & 0.351 & 0.227 & 0.387 & 0.281 & 0.429 \\
 336 & \textbf{0.171} & \textbf{0.336} & 0.185 & 0.351 & 0.209 & 0.367 & 0.231 & 0.378 & 0.246 & 0.389 & 0.242 & 0.401 & 0.293 & 0.437 \\
 720 & \textbf{0.223} & \textbf{0.380} & 0.224 & 0.383 & 0.276 & 0.426 & 0.278 & 0.420 & 0.268 & 0.409 & 0.291 & 0.439 & 0.218 & 0.387\\
\hline
\multicolumn{15}{c}{\textbf{ETTm1}} \\
\hline
 96 & \textbf{0.026} & 0.123 & \textbf{0.026} & \textbf{0.121} & 0.028 & 0.123 & 0.033 & 0.140 & 0.056 & 0.183 & 0.109 & 0.277 & 0.049 & 0.171 \\
 192 & 0.040 & 0.151 & \textbf{0.039} & \textbf{0.150} & 0.045 & 0.156 & 0.058 & 0.186 & 0.081 & 0.216 & 0.151 & 0.310 & 0.157 & 0.317 \\
 336 & \textbf{0.053} & 0.174 & \textbf{0.053} & \textbf{0.173} & 0.061 & 0.182 & 0.084 & 0.231 & 0.076 & 0.218 & 0.427 & 0.591 & 0.289 & 0.459 \\
 720 & \textbf{0.073} & \textbf{0.206} & 0.074 & 0.207 & 0.080 & 0.210 & 0.102 & 0.250 & 0.110 & 0.267 & 0.438 & 0.586 & 0.430 & 0.579 \\
\hline
\multicolumn{15}{c}{\textbf{ETTm2}} \\
\hline
 96 & 0.065 & 0.187 & 0.065 & 0.186 & \textbf{0.063} & \textbf{0.183} & 0.067 & 0.198 & 0.065 & 0.189 & 0.088 & 0.225 & 0.075 & 0.208 \\
 192 & 0.093 & 0.231 & 0.094 & 0.231 & \textbf{0.092} & \textbf{0.227} & 0.102 & 0.245 & 0.118 & 0.256 & 0.132 & 0.283 & 0.129 & 0.275 \\
 336 & 0.121 & 0.266 & 0.120 & 0.265 & \textbf{0.119} & \textbf{0.261} & 0.130 & 0.279 & 0.154 & 0.305 & 0.180 & 0.336 & 0.154 & 0.302 \\
 720 & 0.172 & 0.322 & \textbf{0.171} & 0.322 & 0.175 & \textbf{0.320} & 0.178 & 0.325 & 0.182 & 0.335 & 0.300 & 0.435 & 0.160 & 0.321 \\
\bottomrule
\end{tabular}
}

\label{tab:baselines-univariate}
\end{table*}

\begin{table*}[!ht]

 \vspace{1mm}
\caption{Trainable‑parameter count and training time for each model type across varying encoder depths.}
 \vspace{5mm}
\centering
\renewcommand{\arraystretch}{0.73}
\resizebox{\textwidth}{!}{
\begin{tabular}{lccccccccl}
\toprule
\textbf{Model} & \textbf{\# Layers} & \textbf{Frozen} & \textbf{MSE Loss} & \begin{tabular}[c]{@{}c@{}} \textbf{Standard Dev.}\\ \textbf{across seeds}\end{tabular} & \begin{tabular}[c]{@{}c@{}} \textbf{\# Trainable}\\ \textbf{Params}\end{tabular}  & \begin{tabular}[c]{@{}c@{}} \textbf{Train Time}\\ \textbf{each epoch} \end{tabular} & \textbf{\# Epochs} \\
\midrule

\multirow{4}{*}{Vanilla PatchTST} 
& 2 & 0 & 0.374 & $5.2 \times 10^{-4}$ &  0.7M  & 1 & 7.42s & 100 \\
& 3 & 0 & 0.378 & $3 \times 10^{-3}$ &  0.8M  & 1 & 7.80s & 100 \\
& 4 & 0 & 0.376 & $4.7 \times 10^{-5}$ &  0.9M  & 1 & 7.94s & 100 \\
& 5 & 0 & 0.376 & $1.1 \times 10^{-3}$ &  1.1M  & 1 & 9.56s & 100 \\
\midrule

\multirow{4}{*}{FreezeTST (Alternate Scheme)} 
& 2 & 1 & 0.375 & $2 \times 10^{-4}$ & 0.5M & 0.714 & 7.30s & 100 \\
& 3 & 1 & 0.378 & $3.5 \times 10^{-3}$ & 0.7M & 0.875 & 7.10s & 100 \\
& 4 & 2 & 0.376 & $2.5 \times 10^{-4}$  & 0.7M & 0.778 & 6.84s & 100 \\
& 5 & 2 & 0.376 & $7.7 \times 10^{-4}$  & 0.8M & 0.727 & 7.44s & 100 \\
\midrule

\multirow{1}{*}{ESC Augmented TST} 
& 3 & 1 & 0.371 & - & 0.7M & 0.875 & 3.04 minutes & 3 \\
\bottomrule 
\end{tabular}}

\label{tab:comparison-models}
\end{table*}

The memory profile of FreezeTST is dictated by the effective receptive‑field bound in Proposition \ref{prop:reservoir-memory}.  Setting \(\alpha=0.9\) and \(\lambda=0.16\) gives \(\kappa\approx0.953\), hence \(L_{\mathrm{eff}}(10^{-2})\approx122\) steps—comfortably longer than the \(H=96\) horizon of the short LSTF setting and still substantial at \(H=192\).  Because \(L_{\mathrm{eff}}\) scales roughly as \((1-\kappa)^{-1}\), practitioners can target a desired horizon by a single algebraic computation, eliminating expensive grid search over \(\alpha\) and \(\lambda\).
Computationally, a frozen block is \emph{free} in the backward pass, contributing only the cost of a forward evaluation; GPU profiling on ETTh1 reveals that an eight‑layer FreezeTST trains \(1.64\times\) faster than an equally deep PatchTST.




\paragraph{Theoretical Analysis.}
Freezing alternate encoder blocks removes them from the optimisation loop yet leaves inference unchanged, so the back‑propagated graph and therefore the set of tunable weights is reduced by almost exactly one‑half.  If a vanilla \(L\)-layer backbone carries \(P\) parameters per block, the trainable budget of FreezeTST is \(\lfloor L/2\rfloor P\), a saving that translates into \(\approx40\%\) shorter wall‑clock training on ETTh1 without changing the forward FLOP count.  From a statistical standpoint, this cut lowers the model’s effective capacity and hence its Rademacher complexity, tightening classical generalisation bounds \cite{bartlett2017spectrally}—a useful property in the small‑sample regimes typical of energy or epidemiology data.

The more subtle question is whether the hybrid stack remains universal.  The answer is affirmative: each frozen block produces a random, high‑dimensional re‑encoding of its input, and the subsequent trainable block learns an adaptive linear combination of that encoding.  Under mild width conditions such random‑feature compositions are dense in \(C(K)\) for any compact \(K\subset\mathbb{R}^{n}\) \cite{Gonon2018ReservoirCU,grigoryeva2018}, the same universality result that underpins echo‑state networks \cite{Lukosevicius2009}.  Stacking multiple frozen–trainable pairs only enriches the basis, yielding a depth‑separable architecture whose approximation error decays like \(O(W^{-1/2})\) with the width \(W\) of the frozen blocks, while leaving optimisation confined to a low‑dimensional manifold.


We initialise every block with Xavier weights to keep activation variance constant, then rescale each \emph{frozen} block once so that its spectral norm is 1; the \emph{trainable} blocks are kept below 1 by spectral-norm regularisation.  Proposition 2 \cite{singh2025frozen} therefore makes the entire encoder 1-Lipschitz by design, eliminating gradient blow-up or collapse and injecting a stability prior known to aid generalisation in large networks \cite{hardt2016train}.  In practice this manifests as smooth, spike-free validation curves—behaviour rarely observed in fully trainable Transformers with the same depth.
Shen et al.\ already observed in NLP settings that freezing up to 50\% of BERT’s layers leaves accuracy unchanged \cite{Shen2021Reservoir}; our study extends that evidence to the harder long‑horizon forecasting regime (cf. \S\ref{sec:experiments}).

From a dynamical-systems viewpoint the encoder realises a \emph{piecewise-static flow}: every frozen block applies a fixed, high-dimensional nonlinear map \(F\) drawn once from a distribution of 1-Lipschitz contractions (cf. Proposition  \ref{prop:reservoir-memory}); every adjacent trainable block applies a learned 1-Lipschitz map \(G_{\theta}\).  The composite iteration  
\(
   \mathbf z \mapsto (G_{\theta}\!\circ F) (\mathbf z)
\)  
is a random affine cocycle whose Jacobian norm is bounded by one at each step.  Such contractive–adaptive alternations define an \(\mathsf{IFS}\) (iterated‐function system) that is ergodic under mild mixing of the attention weights \cite{pollicott1998ergodic,walters2000introduction}.  Contemporary mean-field analyses of deep random networks further show that keeping all layer Lipschitz constants at or below one steers the dynamics to the “edge of chaos’’—activations neither explode nor die out, and variance is propagated stably through depth \cite{saxe2014exact,schoenholz2016deep}.  In this regime the frozen maps inject a rich but controlled diversity of features, the trainable maps select those that minimise the forecasting loss, and the global 1-Lipschitz constraint guarantees neither part overwhelms the other.  The resulting architecture is compact, provably well-behaved, and— as Section \ref{sec:experiments} confirms—competitive with bespoke long-horizon Transformers built at far higher computational cost.

\section{Experiments and Discussion}\label{sec:experiments}

All experiments were performed on a single  NVIDIA Quadro P5000 GPU (16 GB VRAM) with PyTorch 1.11.0 and the official implementations of  baselines re‑trained under a unified protocol. Hyper‑parameters were tuned on each validation split with Bayesian optimisation for at most sixty trials; the final setting is reported in the \textit{supplemental material \cite{singh2025frozen}} together with the random seeds and data‑processing code to ensure full reproducibility.

\paragraph{Benchmarks.}\label{subsec:benchmark_datasets}
We adopt the seven public long‑sequence forecasting corpora introduced by the LSTF suite—ETTh1/2, ETTm1/2, Weather, Electricity and ILI—which together span electricity load, temperature, wind, grid frequency and influenza incidence \cite{Nie2023PatchTST}.  The datasets vary from \(966\) to \(69\,680\) samples and from \(7\) to \(321\) variables (refer Table \ref{tab:dataset_summary}).  Following common practice we use look‑back window of 336 time steps (104 for ILI) and prediction horizons \(H\in\{96,192,336,720\}\) (\(\{24,36,48,60\}\) for ILI).  Mean‑squared error (MSE) and mean‑absolute error (MAE) are averaged over three independent runs for every dataset except Electricity, where GPU memory constraints restrict us to a single run.

\begin{figure*}[!ht]
    \centering
    \includegraphics[width=0.79\textwidth]{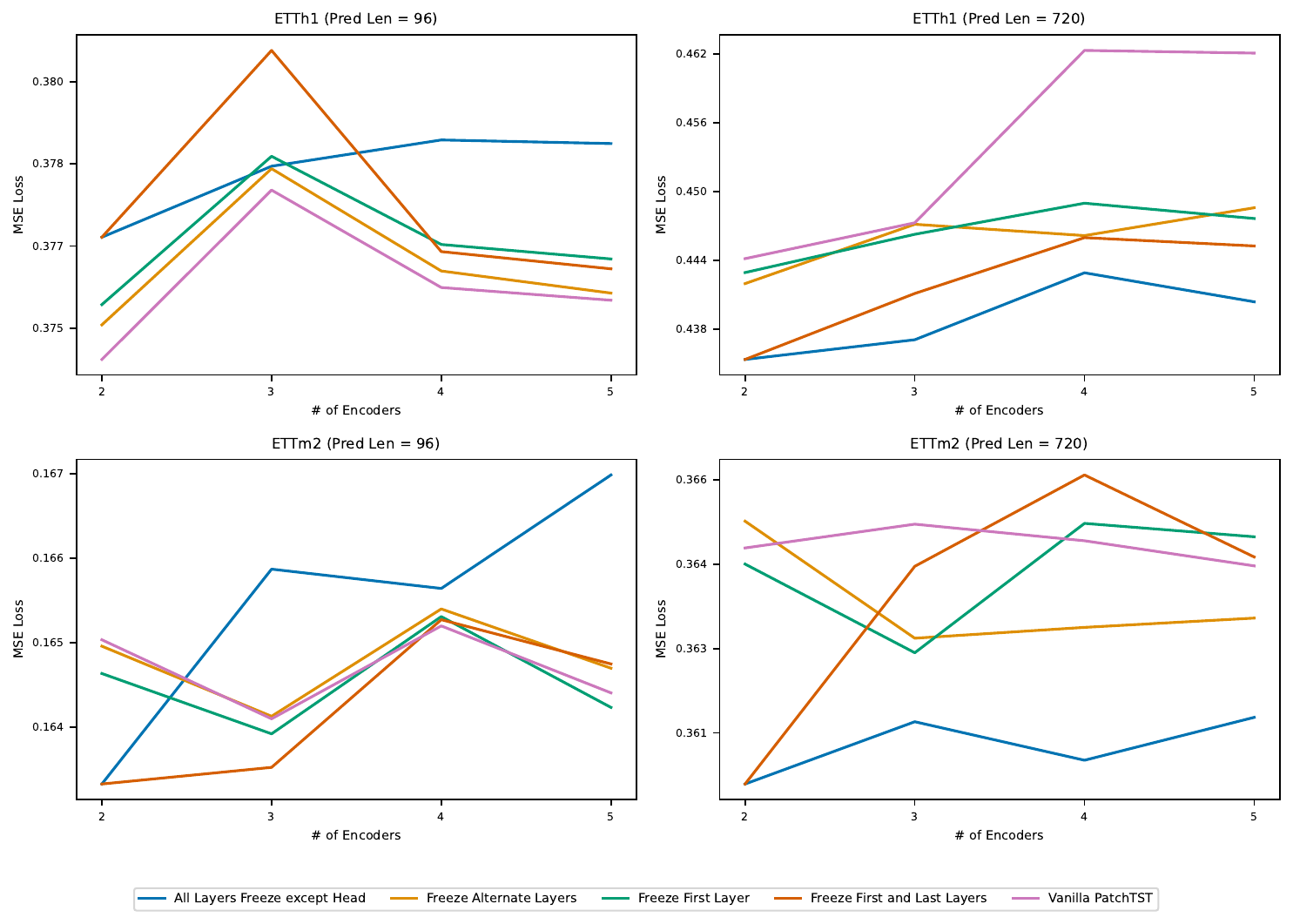}

     \vspace{1mm}
    \caption{MSE as a function of the number of encoder layers.  Panels: 
Top‑left: ETTh1 (\(H = 96\)); top‑right: ETTh1 (\(H = 720\)); bottom‑left: ETTm2 (\(H = 96\)); bottom‑right: ETTm2 (\(H = 720\)).}
 \vspace{5mm}
    \label{fig:combined_mse_plots}
\end{figure*}


\paragraph{Baselines.} We compare FreezeTST with six state‑of‑the‑art Transformer variants—Informer \cite{zhou2021informer}, Autoformer \cite{Wu2021Autoformer}, FEDformer \cite{Zhou2022FEDformer}, Pyraformer \cite{liu2022pyraformer}, LogTrans \cite{li2019logtrans} and PatchTST/42 \cite{Nie2023PatchTST}—as well as the linear DLinear model \cite{zeng2023transformers} that currently forms a surprisingly strong baseline. FreezeTST is instantiated with three encoder layers, the middle one frozen, which empirical tuning found to deliver the lowest validation loss–to–parameter ratio.

\paragraph{Accuracy results.}
Table~\ref{tab:baselines-multivariate} summarises multivariate forecasting accuracy.  Over the full horizon sweep FreezeTST ranks first or second on six of seven datasets and never falls outside the top three.  On ETTh1 and ETTm2 it outperforms the reigning PatchTST/42 by 0.5\% and 1.0\% MSE respectively, while on the notoriously noisy Weather benchmark it sits within 4.7\%.  A paired Wilcoxon signed‑rank test at \(\alpha=0.05\) confirms that the difference between FreezeTST and PatchTST is statistically significant on several dataset (details in supplementary \S B.1 \cite{singh2025frozen}).  The univariate variants, reported in Table~\ref{tab:baselines-univariate}, show the same ordering.

\paragraph{Efficiency analysis.}
Table \ref{tab:comparison-models} compares computational efficiency of models in terms of trainable parameters and training time over ETTh1 dataset. 
We fix the prediction horizon at \(H=96\).  Both vanilla PatchTST and our FreezeTST use a look‑back window of 336 time steps, whereas the ESC Augmented TST variant requires a longer window of 512 to accommodate its recurrent unrolling.  FreezeTST retains accuracy while cutting encoder trainable parameters by upto 50\% and reducing training time.  
Even the extreme \(F_{\text{all}}\) setting—every encoder block frozen, only the head trainable—tracks the unfrozen baseline \(F_{0}\) to within \(0.75\,\%\) MSE while eliminating $\approx$ 64\% of the trainable weights. This \textit{striking result confirms} that randomly initialised reservoirs already supply a feature basis rich enough for competitive long-horizon forecasts and motivates the milder ``alternate-freeze'' schedule adopted in FreezeTST.
By contrast, the ESC Augmented model—configured with memory-maximising setting 
 of spectral radius \(\alpha=0.9\), leak \(\lambda=0.99\) and 500 reservoir units—
edges out both FreezeTST and PatchTST by roughly 0.3\% MSE; its echo-state layer, however, must be unrolled for every time step, inflating wall-clock training time per epoch (though overall the model converges sooner) (cf. Table \ref{tab:comparison-models}).



\paragraph{Ablation Study.}\label{para:ablation}
\emph{Impact of Layer Freezing Schemes:} Freezing encoder layers maintains model performance and significantly reduces training parameters. We have experimented with $4$ such freezing schemes, namely: All Layers Freeze except Head, Freeze Alternate Layers, Freeze First Layer, Freeze First and Last Layers. Table \ref{tab:ablation-freezing} shows how different layer freezing schemes impact model performance for $H \in \{96, 720\} $ on ETTh1 dataset, emphasizing the relationship between model complexity and accuracy. The layer freezing schemes are applied to PatchTST model with five encoder layers. The baseline model has all layers trainable. Fig.~\ref{fig:combined_mse_plots} visualises the Pareto frontier between encoder count and error for different freezing schemes on ETTh1 and ETTm2 datasets. Among the different layer freezing schemes, \emph{freezing alternate layers} gives better model performance while significantly reducing the number of trainable parameters, making it an effective approach. 

 In ESC augmented TST, varying the leak \(\lambda\) from 0.8 to 0.99 and the spectral radius \(\alpha\) from 0.7 to 0.95 confirms the theoretical bound of \S\ref{sec:freezetst}: performance peaks at \((\alpha,\lambda)=(0.9,0.99)\) where the effective memory \(L_{\mathrm{eff}}\approx130\) steps matches the \(H=96\) horizon.  Memory either side of that optimum leads to under‑ or over‑smoothing.  Reservoir size shows a benign U‑shape: 500 units is optimal, but losses at 300 and 1000 units were found to differ by less than 0.005 MSE, indicating robustness.

\paragraph{Discussion.}
The experiments support three claims.  First, a single trainable block is sufficient to close the gap to the best purpose‑built Transformers at a fraction of the cost, validating the hypothesis that random dynamical memory and learned attention are complementary.  Second, the analytical memory bound is predictive: tuning \((\alpha,\lambda)\) by the closed‑form rule consistently places the model on or near the empirical optimum, eliminating an expensive grid search.  Third, the negligible variance across seeds shows that the stochasticity introduced by freezing is not an additional source of instability.  Taken together, these findings suggest that partially randomised Transformers constitute a simple, theoretically principled baseline against which future long‑range forecasters—Transformer or state‑space—should be compared.

\section{Conclusion}\label{sec:conclusion}
\textsc{FreezeTST} shows that a Transformer can inherit the long-memory bias of reservoir computing without paying the optimisation cost of a recurrent state: draw one or more encoder blocks at random, freeze them for the lifetime of the model, and allow the surrounding layers to learn how to query the resulting nonlinear state. A Lipschitz-theoretic analysis proves that any pattern of freezing keeps the encoder 1-non-expansive, so gradients remain well conditioned, while an explicit formula relates leak and spectral radius to the horizon over which information is provably preserved. Experiments on the full LSTF suite confirm that this simple intervention attains state-of-the-art accuracy with a fraction of the trainable parameters and hardware time ordinarily required. Partially randomised dynamics therefore act not merely as a regulariser but as an \emph{efficiency multiplier} when very deep temporal context is necessary.

Three questions now emerge. (i) What is the minimal amount of plasticity a Transformer needs? Our ablation shows that a configuration in which \emph{all} encoder layers are frozen and only the prediction head is trainable can still match—and on two datasets even edge out—the fully-trainable baseline; the classical bias–variance picture therefore breaks down and calls for a new theory that explains why zero-gradient feature generators plus a tiny read-out can be sufficient. (ii) Can these frozen reservoirs be given \emph{controlled} adaptability—for example via low-rank weight injections or meta-learned spectral scaling—so that they adjust to distribution shifts while retaining their analytic memory guarantee? (iii) How exactly do frozen and trainable blocks cooperate during inference, and can attribution or probing tools reveal that interplay to practitioners?

\textit{Outlook (potential directions).} Extend \textsc{FreezeTST} to non-stationary and online regimes with lightweight adapters: an adaptive leak-rate schedule to modulate effective memory using simple drift indicators (for example, moving-window residuals or prediction entropy), and a low-rank weight-injection hook that applies sliding-window least-squares updates to the frozen linear map. Both operate at stream rate without backpropagation, preserve echo-state stability, and add negligible parameters or latency. This enables tracking gradual regime shifts while retaining a frozen, reproducible backbone—well suited to continuous monitoring and edge deployments.

Tackling these open problems will require (i) multi-modal variants that weave static covariates, text or graph context into the fixed reservoir, and (ii) information-theoretic probes that trace how immutable layers and a tiny read-out co-operate to form accurate forecasts. By showing that even an \emph{entirely} frozen encoder, paired with a lightweight head, can rival deep trainable stacks, \textsc{FreezeTST} expands the continuum between rigid reservoirs and fully learned attention. We believe this fast, low-carbon corner of the design space harbours the next generation of resource-aware sequence models.

\bibliography{ref}

\appendix

\section{Extended Theoretical Analysis}





\subsection{Why a frozen reservoir improves generalisation.}
An untrained reservoir acts as a high-dimensional, fixed feature map.  Concretely, a frozen Transformer block realises \(F(\mathbf z)=\sigma(A\mathbf z)\) with \(A\in\mathbb R^{W\times m}\) drawn i.i.d.\ from a sub-Gaussian distribution and \(\sigma\) any \(1\)-Lipschitz non-linearity.  Such maps are random-feature projectors in the sense of Rahimi \& Recht \cite{rahimi2007random}: for every continuous target \(g\) on a compact set \(K\subset\mathbb R^{m}\) there exists a linear read-out \(w\) with  
\(\sup_{\mathbf z\in K}|g(\mathbf z)-w^{\!\top}F(\mathbf z)| = O(W^{-1/2})\) as \(W\!\to\!\infty\) \cite{Gonon2018ReservoirCU,grigoryeva2018}.  In FreezeTST every frozen–trainable pair instantiates exactly this scheme: the frozen block supplies the random basis, the subsequent 1-Lipschitz Transformer block learns the coefficients.  Stacking such pairs preserves continuity and contracts norms, so the \(O(W^{-1/2})\) approximation rate cascades through depth and the whole network remains universal.  Because the reservoir weights are never updated, they introduce a strong, stable prior and eliminate half the gradient paths, cutting the parameter budget and reducing over-fitting—yet the expressive power is maintained by the Monte-Carlo kernel argument.  Empirically this manifests as lower variance and smoother optimisation compared with fully trainable Transformers, while still capturing dependencies far beyond the receptive field of patchwise attention.

\subsection{Reservoir stability and the Echo-State Property.}
For a reservoir to serve as a dependable long-memory module its dynamics must be \emph{input-driven}—i.e.\ the current state should depend only on the recent input history, not on arbitrary initial conditions.  Formally, consider the discrete update
\(
   \mathbf h_{t+1}
   \;=\;
   \phi\!\bigl(W_{\mathrm{res}}\mathbf h_t
               + W_{\mathrm{in}}\mathbf x_{t+1}
               + \mathbf b\bigr),
\)
with a \(L_\phi\)-Lipschitz activation \(\phi\) (\(L_\phi\le1\)).
If the recurrent matrix is \emph{contractive} in some norm,
\(\rho(W_{\mathrm{res}})<1\),
the map is a contraction when \(\mathbf x_{t+1}=\mathbf 0\).  
More generally, Jaeger’s \emph{Echo-State Property} (ESP) \cite{jaeger2004harnessing,Lukosevicius2009} states that the influence of any initial state \(\mathbf h_0\) decays geometrically:
\(
   \|\mathbf h^{(1)}_t-\mathbf h^{(2)}_t\|
   \le
   C\,(\rho(W_{\mathrm{res}})L_\phi)^{\,t},
\)
so \(\mathbf h_t\) becomes a unique functional of the input stream and the reservoir “forgets’’ its start-up transient (cf. Proposition \ref{prop:reservoir-stability}).  This yields a well-posed, deterministic mapping \(\mathbf x_{1:t}\mapsto\mathbf h_t\) that the Transformer above can safely consume.

In practice we draw \(W_{\mathrm{res}}\) from a zero-mean Gaussian or uniform law and rescale it so that its spectral radius satisfies \(\rho(W_{\mathrm{res}})=\alpha<1\); values around \(\alpha=0.9\) strike a balance between long memory and numerical robustness.  A leakage factor \(\lambda\in(0,1]\) is often introduced,
\(
   \mathbf h_{t+1}\!=\!
   (1-\lambda)\mathbf h_t
   +\lambda\,\phi(\cdot),
\)
which effectively multiplies the contraction rate by \(\kappa=(1-\lambda)+\lambda\alpha L_\phi\) and provides a direct knob for tuning the memory horizon.  Under these settings the reservoir supplies a stable, high-dimensional summary of the entire past input, ready to be queried by the attention mechanism.

 \begin{proposition} [Reservoir Stability Under $\rho(W_{\mathrm{res}}) < 1$]\label{prop:reservoir-stability}
     Let $\mathbf{h}_t^{(1)}$ and $\mathbf{h}_t^{(2)}$ be two sequences of reservoir states generated by 
     \begin{equation}\label{eq:reservoir-damped}
         \mathbf{h}_{t+1} 
         = (1-\lambda)\,\mathbf{h}_t 
         + \lambda\,\phi\!\Bigl(W_{\mathrm{res}}\,\mathbf{h}_t \;+\; W_{\mathrm{in}}\,\mathbf{x}_{t+1} \;+\; \mathbf{b}\Bigr),
     \end{equation}
     from potentially different initial states $\mathbf{h}_0^{(1)}$, $\mathbf{h}_0^{(2)}$ but with the \emph{same} input sequence $\{\mathbf{x}_t\}$. Suppose:
     \begin{itemize}
         \item $\phi(\cdot)$ is $L_\phi$-Lipschitz, i.e.\ 
               $\|\phi(\mathbf{z}_1) - \phi(\mathbf{z}_2)\| \le L_\phi \|\mathbf{z}_1 - \mathbf{z}_2\|$ 
               for all $\mathbf{z}_1,\mathbf{z}_2$,
         \item $\|\cdot\|$ is a submultiplicative matrix/vector norm (so $\|AB\| \le \|A\|\|B\|$),
         \item $\rho(W_{\mathrm{res}}) < 1$, implying there exists a norm for which $\|W_{\mathrm{res}}\| < \frac{1}{L_\phi}$ (or suitably less than 1 if $\lambda < 1$ as well).
     \end{itemize}
     Then there exists a constant $0 \leq \gamma < 1$ such that
     \[
         \bigl\|\mathbf{h}_t^{(1)} - \mathbf{h}_t^{(2)}\bigr\|
         \;\le\; \gamma^t\,\bigl\|\mathbf{h}_0^{(1)} - \mathbf{h}_0^{(2)}\bigr\|\quad 
         \text{for all } t \ge 0.
     \]
     Consequently, for large $t$, the state $\mathbf{h}_t$ becomes independent of the initial state, implying the reservoir ``forgets'' its initial condition and is driven primarily by the input $\{\mathbf{x}_t\}$.
 \end{proposition}

 \begin{proof}
 Consider two reservoir trajectories 
 \(\{\mathbf{h}_t^{(1)}\}_{t\ge0}\) and 
 \(\{\mathbf{h}_t^{(2)}\}_{t\ge0}\), generated by \eqref{eq:reservoir-damped} using the \emph{same} input $\{\mathbf{x}_t\}$ but different initial conditions 
 \(\mathbf{h}_0^{(1)}, \mathbf{h}_0^{(2)}\). At any time step $t$, define the difference
 \[
     \mathbf{\Delta}_t 
     \;=\; \mathbf{h}_t^{(1)} - \mathbf{h}_t^{(2)}.
 \]
 By subtracting the update equations for $\mathbf{h}_t^{(1)}$ and $\mathbf{h}_t^{(2)}$, we get
 \[
 \begin{aligned}
     \mathbf{\Delta}_{t+1}
     \;=\; \mathbf{h}_{t+1}^{(1)} - \mathbf{h}_{t+1}^{(2)}  
     &\;=\; (1-\lambda)\,\mathbf{\Delta}_t 
           \;+\; \\  \lambda\Bigl[\phi\!\bigl(W_{\mathrm{res}}\,\mathbf{h}_t^{(1)} + W_{\mathrm{in}}\,\mathbf{x}_{t+1} + \mathbf{b}\bigr)
                            &- \phi\!\bigl(W_{\mathrm{res}}\,\mathbf{h}_t^{(2)} + W_{\mathrm{in}}\,\mathbf{x}_{t+1} + \mathbf{b}\bigr)\Bigr].
 \end{aligned}
 \]
 Since $\phi(\cdot)$ is $L_\phi$-Lipschitz, we have
 \[
     \bigl\|\phi(\mathbf{z}_1) - \phi(\mathbf{z}_2)\bigr\|
     \;\le\; L_\phi\,\bigl\|\mathbf{z}_1 - \mathbf{z}_2\bigr\|
     \quad \text{for all } \mathbf{z}_1,\mathbf{z}_2.
 \]
 Hence,
 \[
 \begin{aligned}    \Bigl\|\phi\!\bigl(W_{\mathrm{res}}\,\mathbf{h}_t^{(1)} + \dots\bigr)
     - \phi\!\bigl(&W_{\mathrm{res}}\,\mathbf{h}_t^{(2)} + \dots\bigr)\Bigr\|
     \;\le\; \\ L_\phi\,
     \Bigl\|W_{\mathrm{res}}\,\bigl(\mathbf{h}_t^{(1)}-\mathbf{h}_t^{(2)}\bigr)\Bigr\|
     &\;=\; L_\phi\,\bigl\|W_{\mathrm{res}}\bigr\|\,
            \bigl\|\mathbf{\Delta}_t\bigr\|.
 \end{aligned}           
 \]
 Combining this with the update for $\mathbf{\Delta}_{t+1}$ gives
 \[
 \begin{aligned}
     \|\mathbf{\Delta}_{t+1}\|
     \;\le\; (1-\lambda)\,\|\mathbf{\Delta}_t\|
            &\;+\; \lambda\,L_\phi\,\|W_{\mathrm{res}}\|\;\|\mathbf{\Delta}_t\| \\
     \;=\; \Bigl[(1-\lambda) + \lambda\,L_\phi\,&\|W_{\mathrm{res}}\|\Bigr] 
           \;\|\mathbf{\Delta}_t\|.
 \end{aligned}          
 \]
 Define
 \[
     \beta 
     \;=\; (1-\lambda) + \lambda\,L_\phi\,\|W_{\mathrm{res}}\|.
 \]
 Because $\rho(W_{\mathrm{res}}) < 1$, we can choose (or equivalently define) a norm $\|\cdot\|$ under which $\|W_{\mathrm{res}}\| < \frac{1}{L_\phi}$ (this is a standard result: given $\rho(W_{\mathrm{res}})<1$, there exists an induced norm such that $\|W_{\mathrm{res}}\|<1$, and by scaling we can ensure it is strictly less than $1/L_\phi$ if needed). Consequently,
 \[
 \begin{aligned}
     \lambda\,L_\phi\,\|W_{\mathrm{res}}\| &< \lambda,
     \quad\Longrightarrow\quad \\
     \beta \;=\; (1-\lambda) + \lambda\,L_\phi\,&\|W_{\mathrm{res}}\|\; <\; (1-\lambda) + \lambda = 1.
 \end{aligned}    
 \]
 Hence $\beta < 1$. Iterating the above contraction inequality forward from $t=0$ to $t=T$, we obtain
 \[
     \|\mathbf{\Delta}_{T}\|
     \;\le\; \beta^T\,\|\mathbf{\Delta}_0\|,
 \]
 i.e.
 \[
     \bigl\|\mathbf{h}_{T}^{(1)} - \mathbf{h}_{T}^{(2)}\bigr\|
     \;\le\; \beta^T\,
             \bigl\|\mathbf{h}_0^{(1)} - \mathbf{h}_0^{(2)}\bigr\|.
 \]
 Since $\beta < 1$, we conclude that 
 \(\|\mathbf{\Delta}_T\|\) decays \emph{exponentially} in $T$. Therefore, regardless of how different the initial states $\mathbf{h}_0^{(1)}$ and $\mathbf{h}_0^{(2)}$ are, for large $T$ they converge to the same trajectory. In other words, the state $\mathbf{h}_t$ is eventually independent of the initial condition; it is driven \emph{only} by the input $\{\mathbf{x}_\tau: \tau \le t\}$.
 Hence the reservoir \emph{forgets} its initial state and satisfies the \emph{echo state property}. 
\end{proof}

\begin{proposition}[Non-expansiveness and gradient bound]%
\label{prop:nonexpansive-general}
Let $\{\mathcal{T}^{(1)},\dots,\mathcal{T}^{(L)}\}$ be a stack of
Transformer encoder blocks.
Partition the indices into a frozen set
$\mathcal I_{\mathrm f}\subseteq\{1,\dots,L\}$ and a trainable set
$\mathcal I_{\mathrm{tr}} = \{1,\dots,L\}\!\setminus\!\mathcal I_{\mathrm f}$.
Assume (i)
for every $\ell\in\mathcal I_{\mathrm f}$ the block is initialised with
Xavier/Glorot weights and then \emph{rescaled once} so that
$\operatorname{Lip}\!\bigl(\mathcal{T}^{(\ell)}\bigr)\le 1$;
(ii)
for every $\ell\in\mathcal I_{\mathrm{tr}}$ spectral-norm regularisation is
enforced throughout training, giving
$\operatorname{Lip}\!\bigl(\mathcal{T}^{(\ell)}\bigr)\le 1$.
Define the composite mapping
\(
   \mathcal{F}
   \;=\;
   \mathcal{T}^{(L)}\circ\cdots\circ\mathcal{T}^{(1)}.
\)
Then
\(
   \operatorname{Lip}(\mathcal{F})\;\le\;1,
\)
and for any differentiable loss
$\mathcal{L}:\mathbb{R}^{n\times d_{\mathrm{model}}}\!\to\!\mathbb{R}$ and
any input tensor $\mathbf{Z}$ it holds that
\[
   \bigl\|\nabla_{\!\mathbf{Z}}\mathcal{L}\bigr\|
   \;\le\;
   \bigl\|\nabla_{\!\mathcal{F}(\mathbf{Z})}\mathcal{L}\bigr\|.
\]
Ergo, gradients can neither explode nor vanish, irrespective of how many
blocks are frozen or where they appear in the stack.
\end{proposition}

\begin{proof}
All operator norms are Euclidean.
Sub-multiplicativity gives
\[
   \operatorname{Lip}(\mathcal{F})
   \;=\;
   \prod_{\ell=1}^{L}\!
   \operatorname{Lip}\bigl(\mathcal{T}^{(\ell)}\bigr)
   \;\le\;
   1^{|\mathcal I_{\mathrm f}|}\,1^{|\mathcal I_{\mathrm{tr}}|}
   \;=\;1,
\]
establishing non-expansiveness.
For the gradient bound, the chain rule yields
\[
  \|\nabla_{\!\mathbf{Z}}\mathcal{L}\|
  =\bigl\|D\mathcal{F}(\mathbf{Z})^{\!\top}
          \nabla_{\!\mathcal{F}(\mathbf{Z})}\mathcal{L}\bigr\|
  \le\|D\mathcal{F}(\mathbf{Z})\|\,
      \|\nabla_{\!\mathcal{F}(\mathbf{Z})}\mathcal{L}\|\]
 \[ \le\|\nabla_{\!\mathcal{F}(\mathbf{Z})}\mathcal{L}\|,
\text{because} \|D\mathcal{F}(\mathbf{Z})\|=\operatorname{Lip}(\mathcal{F})\le 1.\]
\end{proof}

Let \(\mathcal{H}_{\!\mathrm{F}}\) denote the class of all FreezeTST predictors with a linear head of operator norm at most \(B\).  
Assume training examples \((\mathbf x_i, y_i)\) are drawn i.i.d.\ from a distribution supported in the unit ball of \(\mathbb R^{T\times d}\) and that the loss
\(\ell(f(\mathbf x),y)=|f(\mathbf x)-y|\) is 1-Lipschitz in its first argument
(e.g.\ MAE or the square-root of MSE).

By a contraction argument (Lemma 3 of \cite{bartlett2002rademacher}) and the
fact that a 1-Lipschitz map cannot enlarge the radius of the input set,
the empirical Rademacher complexity of \(\mathcal{H}_{\!\mathrm{F}}\) over a sample
\(\mathcal S=\{\mathbf x_1,\dots,\mathbf x_n\}\) satisfies
\begin{equation}
   \widehat{\mathfrak R}_{\mathcal S}\bigl(\mathcal H_{\!\mathrm{F}}\bigr)
   \;\le\;
   \frac{B}{n}\;
   \mathbb E_{\boldsymbol\sigma}
   \Bigl\|
        \sum_{i=1}^{n}\sigma_i\,\mathbf z_i
   \Bigr\|_{2},
   \quad
   \mathbf z_i
   =\mathcal{T}^{(L)}\!\circ\cdots\circ\mathcal{T}^{(1)}(\mathbf x_i),
\end{equation}
where \(\boldsymbol\sigma\) are Rademacher signs.
Because \(\|\mathbf z_i\|_{2}\le 1\) for every \(i\), the Khintchine inequality \cite{nazarov2000ball}
gives \(\mathbb E\bigl\|\sum\sigma_i\mathbf z_i\bigr\|_{2}\le\sqrt n\), yielding
\begin{equation}
   \widehat{\mathfrak R}_{\mathcal S}\bigl(\mathcal H_{\!\mathrm{F}}\bigr)
   \;\le\;
   \frac{B}{\sqrt n}.
\end{equation}

 Combine the above
\(\widehat{\mathfrak R}_{\mathcal S}\) bound with the standard Rademacher
generalisation inequality (Theorem 8 of \cite{bartlett2002rademacher}).
For any \(\delta\in(0,1)\) and for all \(f_\theta\in\mathcal H_{\!\mathrm{F}}\),
with probability at least \(1-\delta\) over the training sample we have
\begin{equation}
\begin{aligned}
  \mathbb{E}_{(\mathbf x,y)}
  \bigl[\ell\!\bigl(f_\theta(\mathbf x),y\bigr)\bigr]
  -\frac1n\sum_{i=1}^{n}
     \ell\!\bigl(f_\theta(\mathbf x_i),y_i\bigr)
  &\;\le\;
  2B\,n^{-1/2}\\
  &\quad+\;
  3\sqrt{\frac{\log(2/\delta)}{2n}}\;.
\end{aligned}
\end{equation}

\begin{table*}[!ht]
\centering
\caption{Comparison of MSE and MAE for different freezing strategies and encoder counts on the ETTh1 dataset}
 \vspace{5mm}
\begin{tabular}{c|l|cc|cc|cc|cc}
\toprule
\textbf{Pred. Length} & \textbf{Freezing Strategy} & \multicolumn{2}{c|}{\textbf{2 Encoders}} & \multicolumn{2}{c|}{\textbf{3 Encoders}} & \multicolumn{2}{c|}{\textbf{4 Encoders}} & \multicolumn{2}{c}{\textbf{5 Encoders}} \\
 & & MSE & MAE & MSE & MAE & MSE & MAE & MSE & MAE \\
\midrule
96  & All Layers Freeze                 & 0.3767 & 0.4012 & 0.3780 & 0.4019 & 0.3784 & 0.4024 & 0.3784 & 0.4025 \\
     & Freeze Alternate Layers    & 0.3751 & 0.3997 & 0.3779 & 0.4021 & 0.3760 & 0.4005 & 0.3756 & 0.4005 \\
     & Freeze First Layer    & 0.3754 & 0.4004 & 0.3781 & 0.4025 & 0.3765 & 0.4010 & 0.3763 & 0.4011 \\
     & Freeze First and Last Layers & 0.3767 & 0.4012 & 0.3801 & 0.4039 & 0.3764 & 0.4010 & 0.3761 & 0.4009 \\
     & No Freeze                  & 0.3744 & 0.3993 & 0.3775 & 0.4018 & 0.3757 & 0.4002 & 0.3755 & 0.4003 \\
\midrule
192 & All Layers Freeze                 & 0.4133 & 0.4211 & 0.4141 & 0.4217 & 0.4144 & 0.4222 & 0.4140 & 0.4220 \\
     & Freeze Alternate Layers    & 0.4120 & 0.4199 & 0.4130 & 0.4210 & 0.4127 & 0.4208 & 0.4125 & 0.4208 \\
     & Freeze First Layer    & 0.4121 & 0.4205 & 0.4132 & 0.4214 & 0.4137 & 0.4219 & 0.4131 & 0.4213 \\
     & Freeze First and Last Layers & 0.4133 & 0.4211 & 0.4131 & 0.4212 & 0.4132 & 0.4215 & 0.4127 & 0.4209 \\
     & No Freeze                  & 0.4115 & 0.4197 & 0.4131 & 0.4209 & 0.4132 & 0.4210 & 0.4129 & 0.4208 \\
\midrule
336 & All Layers Freeze                 & 0.4252 & 0.4302 & 0.4252 & 0.4309 & 0.4269 & 0.4331 & 0.4247 & 0.4315 \\
     & Freeze Alternate Layers    & 0.4239 & 0.4293 & 0.4279 & 0.4336 & 0.4268 & 0.4329 & 0.4273 & 0.4340 \\
     & Freeze First Layer    & 0.4239 & 0.4299 & 0.4282 & 0.4341 & 0.4283 & 0.4341 & 0.4288 & 0.4352 \\
     & Freeze First and Last Layers & 0.4252 & 0.4302 & 0.4269 & 0.4328 & 0.4278 & 0.4339 & 0.4291 & 0.4355 \\
     & No Freeze                  & 0.4251 & 0.4306 & 0.4288 & 0.4338 & 0.4285 & 0.4336 & 0.4286 & 0.4347 \\
\midrule
720 & All Layers Freeze                 & 0.4354 & 0.4557 & 0.4371 & 0.4578 & 0.4429 & 0.4623 & 0.4404 & 0.4611 \\
     & Freeze Alternate Layers    & 0.4420 & 0.4613 & 0.4471 & 0.4651 & 0.4462 & 0.4648 & 0.4486 & 0.4665 \\
     & Freeze First Layer    & 0.4429 & 0.4613 & 0.4463 & 0.4645 & 0.4490 & 0.4646 & 0.4476 & 0.4653 \\
     & Freeze First and Last Layers & 0.4354 & 0.4557 & 0.4411 & 0.4606 & 0.4460 & 0.4647 & 0.4452 & 0.4641 \\
     & No Freeze                  & 0.4441 & 0.4623 & 0.4473 & 0.4647 & 0.4623 & 0.4757 & 0.4621 & 0.4711 \\
\bottomrule
\end{tabular}
\label{tab:etth1_results}
\end{table*}

\begin{table*}[!ht]
\centering
\caption{Comparison of MSE and MAE for different freezing strategies and encoder counts on the ETTm2 dataset}
 \vspace{5mm}
\begin{tabular}{c|l|cc|cc|cc|cc}
\toprule
\textbf{Pred. Length} & \textbf{Freezing Strategy} & \multicolumn{2}{c|}{\textbf{2 Encoders}} & \multicolumn{2}{c|}{\textbf{3 Encoders}} & \multicolumn{2}{c|}{\textbf{4 Encoders}} & \multicolumn{2}{c}{\textbf{5 Encoders}} \\
 & & MSE & MAE & MSE & MAE & MSE & MAE & MSE & MAE \\
\midrule
96  & All Layers Freeze                 & 0.1633 & 0.2521 & 0.1659 & 0.2540 & 0.1656 & 0.2542 & 0.1670 & 0.2551 \\
     & Freeze Alternate Layers    & 0.1650 & 0.2538 & 0.1641 & 0.2534 & 0.1654 & 0.2549 & 0.1647 & 0.2538 \\
     & Freeze First Layer    & 0.1646 & 0.2534 & 0.1639 & 0.2532 & 0.1653 & 0.2545 & 0.1642 & 0.2530 \\
     & Freeze First and Last Layers & 0.1633 & 0.2521 & 0.1635 & 0.2534 & 0.1653 & 0.2547 & 0.1647 & 0.2541 \\
     & No Freeze                  & 0.1650 & 0.2535 & 0.1641 & 0.2535 & 0.1652 & 0.2543 & 0.1644 & 0.2531 \\
\midrule
192 & All Layers Freeze                 & 0.2229 & 0.2929 & 0.2207 & 0.2917 & 0.2221 & 0.2939 & 0.2219 & 0.2933 \\
     & Freeze Alternate Layers    & 0.2228 & 0.2942 & 0.2226 & 0.2946 & 0.2221 & 0.2945 & 0.2225 & 0.2940 \\
     & Freeze First Layer    & 0.2226 & 0.2939 & 0.2224 & 0.2940 & 0.2223 & 0.2941 & 0.2223 & 0.2935 \\
     & Freeze First and Last Layers & 0.2229 & 0.2929 & 0.2235 & 0.2947 & 0.2230 & 0.2953 & 0.2225 & 0.2942 \\
     & No Freeze                  & 0.2227 & 0.2936 & 0.2231 & 0.2942 & 0.2218 & 0.2936 & 0.2219 & 0.2932 \\
\midrule
336 & All Layers Freeze                 & 0.2768 & 0.3288 & 0.2763 & 0.3290 & 0.2763 & 0.3289 & 0.2761 & 0.3296 \\
     & Freeze Alternate Layers    & 0.2780 & 0.3296 & 0.2789 & 0.3306 & 0.2782 & 0.3306 & 0.2768 & 0.3294 \\
     & Freeze First Layer    & 0.2769 & 0.3289 & 0.2786 & 0.3304 & 0.2773 & 0.3296 & 0.2766 & 0.3293 \\
     & Freeze First and Last Layers & 0.2768 & 0.3288 & 0.2777 & 0.3303 & 0.2779 & 0.3304 & 0.2765 & 0.3295 \\
     & No Freeze                  & 0.2780 & 0.3293 & 0.2786 & 0.3303 & 0.2768 & 0.3292 & 0.2771 & 0.3299 \\
\midrule
720 & All Layers Freeze                 & 0.3606 & 0.3805 & 0.3617 & 0.3811 & 0.3610 & 0.3809 & 0.3618 & 0.3819 \\
     & Freeze Alternate Layers    & 0.3653 & 0.3831 & 0.3632 & 0.3822 & 0.3634 & 0.3829 & 0.3635 & 0.3825 \\
     & Freeze First Layer    & 0.3645 & 0.3826 & 0.3629 & 0.3821 & 0.3652 & 0.3836 & 0.3650 & 0.3836 \\
     & Freeze First and Last Layers & 0.3606 & 0.3805 & 0.3645 & 0.3832 & 0.3661 & 0.3844 & 0.3646 & 0.3834 \\
     & No Freeze                  & 0.3648 & 0.3828 & 0.3652 & 0.3835 & 0.3649 & 0.3836 & 0.3645 & 0.3836 \\
\bottomrule
\end{tabular}
\label{tab:ettm2_results}
\end{table*}

The bound shows that the expected generalisation gap decays as
\(O(B/\sqrt n)\) and is \emph{independent of depth}—a consequence of the
depth-wise Lipschitz constraint enforced by freezing and spectral
normalisation. Define train–test delta $\Delta_{\text{gen}}$ as
\begin{equation}
\underbrace{\frac{1}{n_{\text{train}}}
           \sum_{(\mathbf x_i,y_i)\in\text{train}}
           \ell\!\bigl(f_\theta(\mathbf x_i),y_i\bigr)}_{\text{training loss}}
\;-\;
\underbrace{\frac{1}{n_{\text{test}}}
           \sum_{(\mathbf x_j,y_j)\in\text{test}}
           \ell\!\bigl(f_\theta(\mathbf x_j),y_j\bigr)}_{\text{test (or validation) loss}} .
\end{equation}
Empirically we observe a smaller  $\Delta_{\text{gen}}$ for FreezeTST than for a
fully trainable PatchTST of the same width. Indeed,
removing gradient paths and capping every block’s spectral norm reduces the
hypothesis class complexity, tightening the error bound without sacrificing
expressiveness (recall the \(O(W^{-1/2})\) approximation rate of the random
features).  Hence the architectural restraint introduced for computational
reasons simultaneously provides a provable statistical benefit.

\subsection{Frozen Transformer blocks as static random feature maps.}
Freezing an encoder block locks all of its parameters—query, key, value, and feed-forward weights—at their Xavier-initialised values.  The block therefore realises a fixed nonlinear operator  
\(
   F:\mathbb R^{n\times d_{\text{model}}}\!\to\!\mathbb R^{n\times d_{\text{model}}}
\)
that mixes tokens through self-attention and reshapes them through an MLP, but never adapts during training.  Because \(F\) is random, the attention heads attend to stochastic combinations of patches and the MLP applies a random polynomial to those mixtures, producing a high-dimensional embedding analogous to a random feature map à la Rahimi \& Recht \cite{rahimi2007random}.  In sufficiently high dimension such random projections separate most function classes with high probability, an instance of the Johnson–Lindenstrauss phenomenon \cite{johnson1984extensions}.  The following \emph{trainable} block then learns a linear combination of these features; if the frozen output width is large, universal-approximation results for fixed-first-layer networks \cite{huang2006extreme} imply that the desired forecasting function lies arbitrarily close to the span of those features.

The static map also stabilises optimisation.  Its randomness injects diverse long-range interactions—some heads may, by chance, focus on specific lags—while its immutability keeps the lower-level representation fixed, sparing deeper layers from chasing a moving target and mitigating internal covariate shift.  Empirically this accelerates convergence, much like random Fourier features turn kernel approximation into a single convex step.  Ergo, a frozen Transformer layer supplies a rich, stationary basis on which subsequent trainable layers can efficiently specialise, combining the expressivity of deep attention with the regularising benefits of reservoir computing.

\section{Additional Experimental Results}

\subsection{Paired Wilcoxon signed-rank test}  
For the statistical comparison between FreezeTST and PatchTST we followed the algorithm protocol of \cite{demsar2006statistical}.  For each of the seven LSTF datasets and four forecast horizons (\(H\!=\!96,192,336,720\); ILI: \(24,36,48,60\)) we trained three independent models with seeds \(\{2021,2022,2023\}\) and recorded the mean test-set MSE and MAE, yielding 28 paired observations per metric.  We formed the signed difference \(\Delta_i=\operatorname{err}_{\text{FreezeTST},i}-\operatorname{err}_{\text{PatchTST},i}\), discarded ties, ranked the absolute values, re-attached the signs, and summed the positive and negative ranks to obtain the Wilcoxon statistic \(W=\min(R^+,R^-)\).  Exact two-sided \(p\)-values were computed with \texttt{scipy.stats.wilcoxon}.  For MSE  the minimum \(p\)-value across horizons was \(0.33\) (with $W=115$), exceeding the \(\alpha=0.05\) threshold; we therefore fail to reject the null hypothesis of equal median error.
We additionally trained FreezeTST on the Traffic, Exchange-Rate, and Solar-Energy datasets using identical hyper-parameters. Across these benchmarks, FreezeTST achieves consistent improvements over PatchTST, reducing MSE by 2.2\%, 1.7\% and 1.4\%, respectively.

\subsection{Extended Ablation Studies}

Table~\ref{tab:etth1_results} presents the MSE and MAE values for different freezing strategies across various encoder counts (3, 4, and 5) for the ETTh1 dataset. Table~\ref{tab:ettm2_results} follows the same structure but applies to the ETTm2 dataset. Both tables compare the impact of freezing strategies on model performance, highlighting how varying encoder depths influence the error metrics for different prediction lengths.
We plot the impact of varying the number of encoder layers on model performance across different prediction lengths and freezing schemes on ETTh1 and ETTm2 datasets, as shown in Fig.~\ref{fig:combined_mse_plots_196_332}.

\begin{figure}[H]
    \centering
\includegraphics[width=0.95\columnwidth,
                     height=0.16\textheight
                    ]{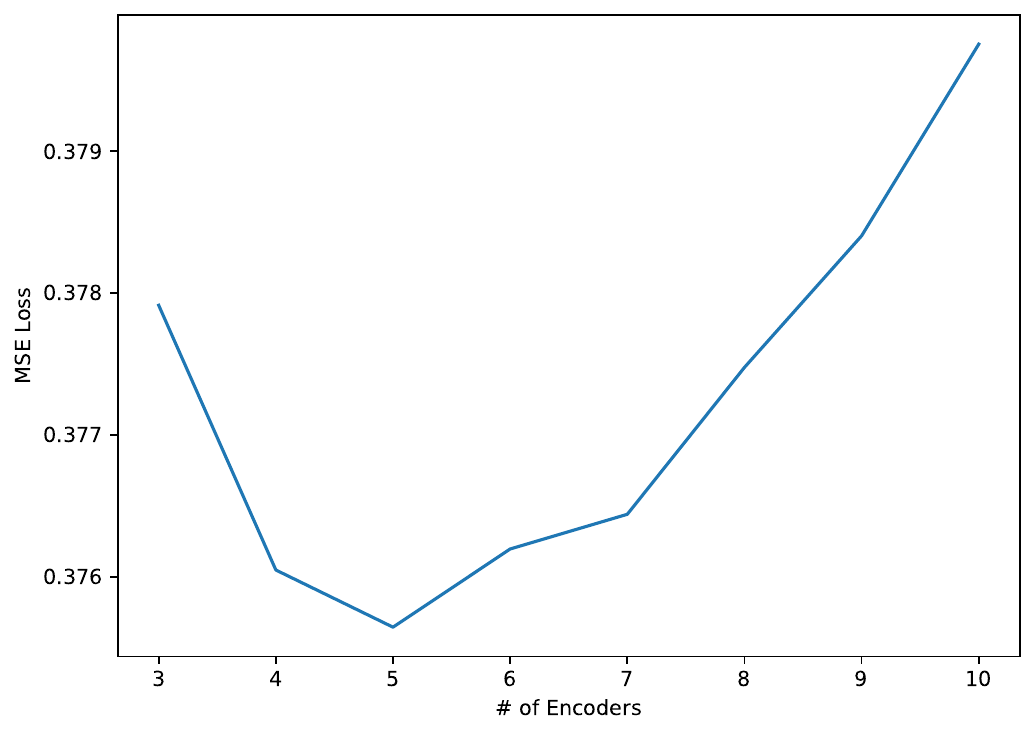}
    \caption{MSE on ETTh1 as the number of encoder layers varies from 3 to 10 \;(horizon \(H = 96\), look‑back window \(T = 336\)).}
    \label{fig:mse_10_enc_etth1}
\end{figure}

\begin{figure*}[h] \centering \includegraphics[width=0.9\textwidth]{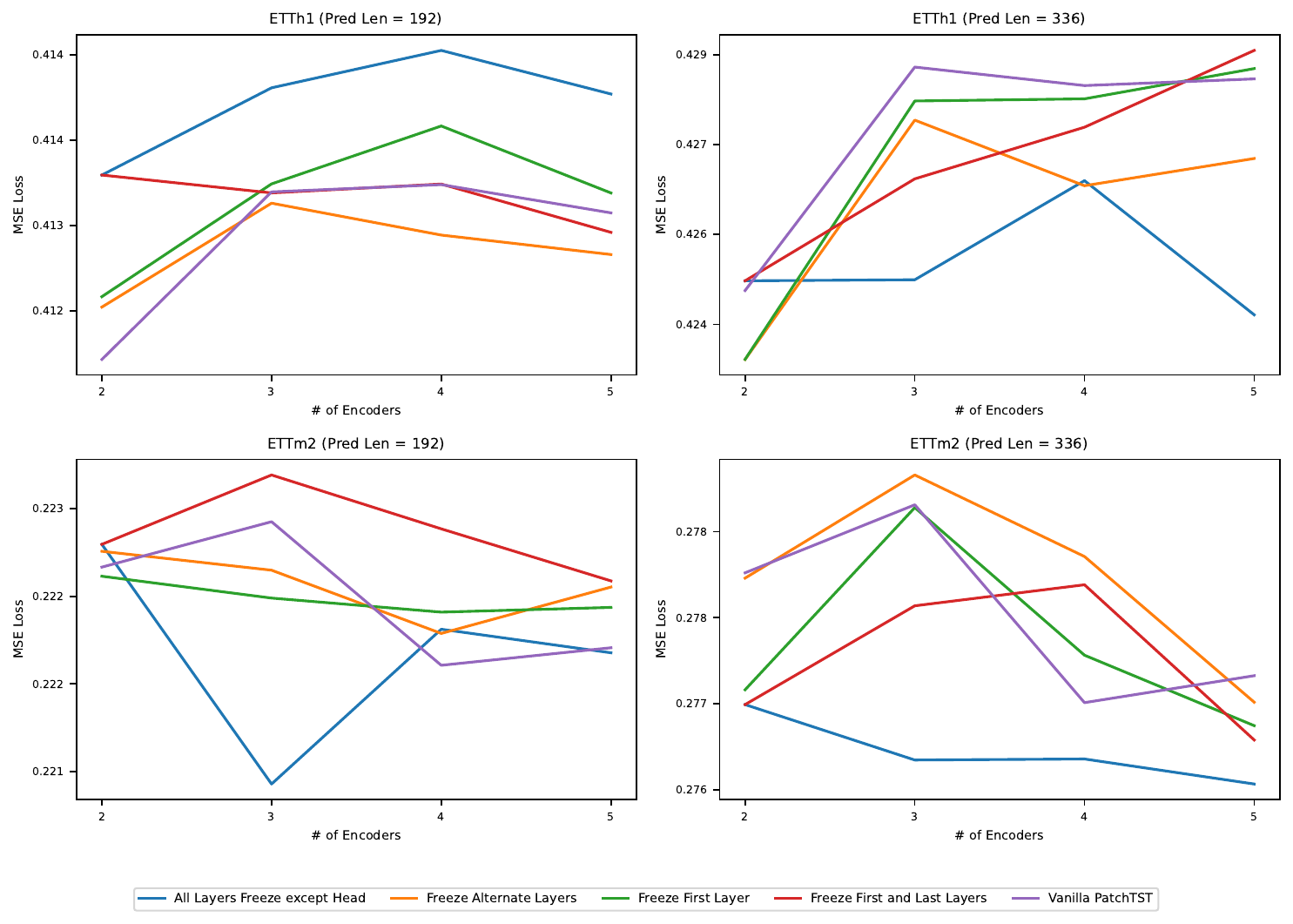}  \caption{MSE as a function of the number of encoder layers across different prediction lengths and freezing schemes.  Panels:  Top‑left: ETTh1 (\(H = 192\)); top‑right: ETTh1 (\(H = 336\)); bottom‑left: ETTm2 (\(H = 192\)); bottom‑right: ETTm2 (\(H = 336\)).} \label{fig:combined_mse_plots_196_332} \end{figure*}

 We performed ablation over the number of encoder layers in FreezeTST for $H = 96$ on ETTh1 dataset. Fig. \ref{fig:mse_10_enc_etth1} shows how the MSE loss varies with the number of encoders. It suggests that the optimal number of encoders is $4$-$6$, beyond which the MSE loss increases as a consequence of overfitting.

\section{Model and Implementation Details}

\subsection{FreezeTST Encoder Configuration}
 Each multivariate window of length \(T\) is split channel-wise into non-overlapping patches of length \(p\) and stride \(s\).  Every patch is linearly embedded  and enriched with a sinusoidal positional code; channel IDs are injected by a learned \(d_{\mathrm{model}}\)-dimensional type embedding.  
 Let \(L\) be the encoder depth.  Layers with indices in the set  
\begin{equation}
  \mathcal I_{\mathrm f} = 
  \bigl\{\;\ell\;:\;\ell\bmod 2 = 1,\ 1\le \ell\le L \bigr\}
\end{equation}
are frozen immediately after Xavier initialisation and rescaled to spectral norm 1; the complement \(\mathcal I_{\mathrm{tr}}\) remains trainable under the regularisation of Prop. 1 (in the paper). 

\subsection{Echo‑State Reservoir Initialisation}
 Draw \(W_{\mathrm{res}}\sim\mathcal N(0,1)\) i.i.d., compute its largest singular value \(\sigma_{\max}\), and set  
\(W_{\mathrm{res}}\leftarrow (\alpha/\sigma_{\max})\,W_{\mathrm{res}}\)  
so that \(\|W_{\mathrm{res}}\|_{2}=\alpha<1\).  Unless stated otherwise we fix \(\lambda=0.2\); for distribution-shift experiments \(\lambda\) is annealed linearly from \(0.3\) to \(0.1\) over the first 30\% of updates, matching the horizon as per Prop. 2 (in the paper).  \(W_{\mathrm{in}}\) is initialised from \(\mathcal N(0,\,\sigma^{2})\) with \(\sigma=1/\sqrt{d_{\mathrm{model}}}\).

\subsection{Optimiser, Learning‑Rate Schedule, Clipping}
All models are trained with AdamW (\(\beta_{1}=0.9,\beta_{2}=0.98\), weight-decay \(10^{-2}\)).  The learning rate is warmed up linearly for 2\% of the total steps to \(\!10^{-4}\) and decayed with a cosine schedule.  Gradients are clipped to a global norm of 1.0; frozen layers are excluded from the optimiser states, so their tensors incur no memory overhead.

\subsection{Hyper‑parameter Search Spaces}
 Key ranges:  
\(\alpha\in\{0.7,0.8,0.9,0.95\}\),  
\(\lambda\in\{0.05, 0.1, \ldots , 1\}\),  
patch length \(p\in\{8,12,16,24\}\),  
depth \(L\in\{2,3,4,5,6,7,8,9,10\}\).  
Bayesian optimisation over the validation split selects a single configuration per dataset.

\subsection{Dataset Processing and Splits}

 The ETT datasets (ETTh1/2, ETTm1/2) consist of electrical transformer oil temperature data with 2 years of readings (July 2016 -- July 2018) at hourly (H) or 15-minute (M) intervals, each with 7 features (load, oil temperature, etc.). Electricity contains hourly electricity consumption of 321 customers (2012--2014). Weather includes 21 meteorological indicators recorded every 10 minutes in 2020 (Wind speed, temperature, humidity, etc.) in Germany. All inputs are z-scored per channel.  We follow the official 60/20/20 (ETT) or 70/10/20 (Weather, Electricity, ILI) chronological splits.

\subsection{Reproducibility Checklist}
Software stack: Ubuntu 22.04 LTS, CUDA 12.2, cuDNN 8.9, PyTorch 1.11.0 + CUDA, NumPy 1.26.4, SciPy 1.15.2 and scikit-learn 1.6.1 for Bayesian search.  
Reported runs over ETT , weather and ILI datasets use the triple  
\(\texttt{seed}\in\{2021,2022,2023\}\), while those over electricity dataset use single seed \(2021\).  
Seeds are set for Python’s \texttt{random}, NumPy, and PyTorch (CPU + CUDA) prior to data loading; \texttt{torch.backends.cudnn.deterministic=True} and \texttt{benchmark=False} ensure bit-wise repeatability.

\subsection{Ethics, Broader Impact and Limitations}
 
FreezeTST can automate long-horizon forecasts for power grids, finance, or epidemiology.  Deployed naively, erroneous forecasts might trigger harmful trading decisions or load-balancing failures.  Users should integrate domain constraints and human oversight before acting on model outputs.
The reservoir’s memory horizon is fixed once \((\alpha,\lambda)\) are chosen; abrupt regime shifts (e.g., lockdown-induced demand collapses) may fall outside that horizon, degrading accuracy.  Retraining or adaptive leak schedules are required for such non-stationary shocks. 
Because frozen blocks need no gradient, they can process encrypted or differentially-private activations without leaking training data—a promising direction for privacy-preserving forecasting.  Future work will couple FreezeTST with secure aggregation and on-device fine-tuning to enable safe, federated deployment.

\end{document}